\documentclass{article}
\pdfoutput=1

\usepackage[final]{neurips_2020}
\usepackage
[acronym,smallcaps,nowarn,section,nogroupskip,nonumberlist]{glossaries}
\glsdisablehyper{}
\glsdisablehyper{}

\newacronym{KL}{KL}{Kullback-Leibler}
\newacronym{ELBO}{elbo}{\emph{evidence lower bound}}
\newacronym{MCMC}{mcmc}{Markov chain Monte Carlo}

\newacronym{ML}{ML}{machine learning}
\newacronym{VAE}{VAE}{variational auto-encoder}
\newacronym{AE}{AE}{auto-encoder}
\newacronym{SGD}{SGD}{stochastic gradient descent}
\newacronym[sort=beta]{BVAE}{\(\beta\)-vae}{}
\newacronym{TC}{TC}{total correlation}
\newacronym{MMD}{MMD}{maximum mean discrepancy}
\newacronym{GAN}{GAN}{generative adversarial network}
\newacronym{AAE}{AAE}{adversarial auto-encoder}
\newacronym{CCM}{CCM}{constant curvature manifold}
\newacronym{IWAE}{IWAE}{importance weighted auto-encoder}
\newacronym{MC}{MC}{Monte Carlo}
\newacronym{ARS}{ARS}{adaptive rejection sampling}
\newacronym{OT}{OT}{optimal transport}
\newacronym{SDE}{SDE}{stochastic differential equation}
\newacronym{ODE}{ODE}{ordinary differential equation}
\newacronym{IVP}{IVP}{initial value problem}
\newacronym{NF}{NF}{normalizing flow}
\newacronym{CNF}{CNF}{continuous normalizing flow}
\newacronym{RCNF}{RCNF}{Riemmanian continuous normalizing flow}
\newacronym{NFE}{NFE}{number of function evaluations}
\newacronym{NN}{NN}{neural network}
\newacronym{FNN}{FNN}{feedforward neural network}
\newacronym{MLP}{MLP}{multilayer perceptron}
\newacronym{SVM}{SVM}{support vector machine}
\newacronym{KDE}{KDE}{kernel density estimation}
\newacronym{RKHS}{RKHS}{reproducing kernel Hilbert space}
\newacronym{VI}{VI}{variational inference}
\newacronym{maxlike}{ML}{maximum likelihood}
\newacronym{RK}{RK}{Runge-Kutta}

\newcommand{\g}{\,|\,}

\DeclareRobustCommand{\KL}[2]{\ensuremath{D_{\textrm{KL}}\left(#1\;\|\;#2\right)}}

\newcommand{\diag}{\textrm{diag}}
\newcommand{\supp}{\textrm{supp}}

\DeclareMathOperator*{\argmin}{arg\,min}

\usepackage[utf8]{inputenc} 
\usepackage[T1]{fontenc}    
\usepackage{hyperref}       
\usepackage[capitalise,noabbrev]{cleveref}
\usepackage{url}            
\usepackage{booktabs}       
\usepackage{amsfonts}       
\usepackage{mathtools}
\usepackage{nicefrac}       
\usepackage{microtype}      
\usepackage{todonotes}      
\usepackage{bm}             
\usepackage{appendix}       
\usepackage{amssymb}        
\usepackage{amsmath}        
\usepackage{amsthm}
\usepackage{txfonts}
\usepackage{esvect}         
\usepackage{bbold}          %
\usepackage{tkz-euclide}    
\usepackage{algpseudocode,algorithm}  
\usepackage[font={small}]{caption, subcaption}
\usepackage{multirow}
\usepackage{wrapfig}
\usepackage{graphicx}

\def\x{\bm{x}}
\def\z{\bm{z}}
\def\y{\bm{y}}

\def\w{\bm{w}}

\def\v{\bm{v}}
\def\f{\bm{f}}
\def\a{\bm{a}}
\def\p{\bm{p}}
\def\r{\bm{r}}
\def\e{\bm{\text{e}}}
\def\M{\mathcal{M}}

\def\T{\mathcal{T}}
\def\R{\mathbb{R}}
\def\B{\mathbb{B}}

\def\S{\mathbb{S}}
\def\g{\mathfrak{g}}
\def\vMF{\text{vMF}}
\DeclareMathOperator{\tr}{tr}
\DeclareMathOperator{\diver}{div}
\DeclareMathOperator{\vol}{Vol}
\DeclareMathOperator{\leb}{Leb}
\def\1{\mathbb{1}}

\renewcommand{\subparagraph}[1]{\paragraph{\normalfont\itshape{#1}}}
\let\cline\cmidrule

\newtheorem{proposition}{Proposition}

\newtheorem{corollary}{Corollary}[proposition]
\newcommand{\appendixhead}{
  \centerline{\textbf{\huge Appendix for}\vspace{0.15in}}
  \centerline{\textbf{\huge Riemannian Continuous Normalizing Flows}\vspace{0.25in}}
  }
\let\origappendix\appendix 
\renewcommand\appendix{\pagenumbering{arabic}\origappendix}
\AtBeginEnvironment{appendices}{\crefalias{section}{appendix}}
\creflabelformat{equation}{#2\textup{#1}#3}
\usepackage{xcolor}
\hypersetup{
  colorlinks,
  linkcolor={red!50!black},
  citecolor={blue!50!black},
  urlcolor={blue!80!black}
  }
  \definecolor{orange}{HTML}{ff7f0e}
  \definecolor{blue}{HTML}{1f77b4}
  \definecolor{green}{HTML}{2ca02c}

\graphicspath{{images/}}

\title{Riemannian Continuous Normalizing Flows}

\author{
  Emile Mathieu$^{\dagger}$\thanks{Work done while at Facebook AI research.}\ , \ 
  Maximilian Nickel$^{\ddagger}$ \\
  \texttt{emile.mathieu@stats.ox.ac.uk}, \
  \texttt{maxn@fb.com} \\
  $^\dagger$ Department of Statistics, University of Oxford, UK\\
  $^\ddagger$ Facebook Artificial Intelligence Research, New York, USA \\
}

\begin{document}

\maketitle

\begin{abstract}
Normalizing flows have shown great promise for modelling flexible probability
distributions in a computationally tractable way. However, whilst data is often
naturally described on Riemannian manifolds such as spheres, tori, and
hyperbolic spaces, most normalizing flows implicitly assume a flat geometry,
making them either misspecified or ill-suited in these situations. To overcome
this problem, we introduce \emph{Riemannian continuous normalizing flows}, a
model which admits the parametrization of flexible probability measures on
smooth manifolds by defining flows as the solution to \acrlongpl{ODE}.
We show that this approach can lead to substantial improvements on both
synthetic and real-world data when compared to standard flows or previously
introduced projected flows.
\end{abstract}


\section{Introduction}
\begin{wrapfigure}{r}{0.33\textwidth}
\vspace{-1.5em}
\centering
  \includegraphics[width=\linewidth]{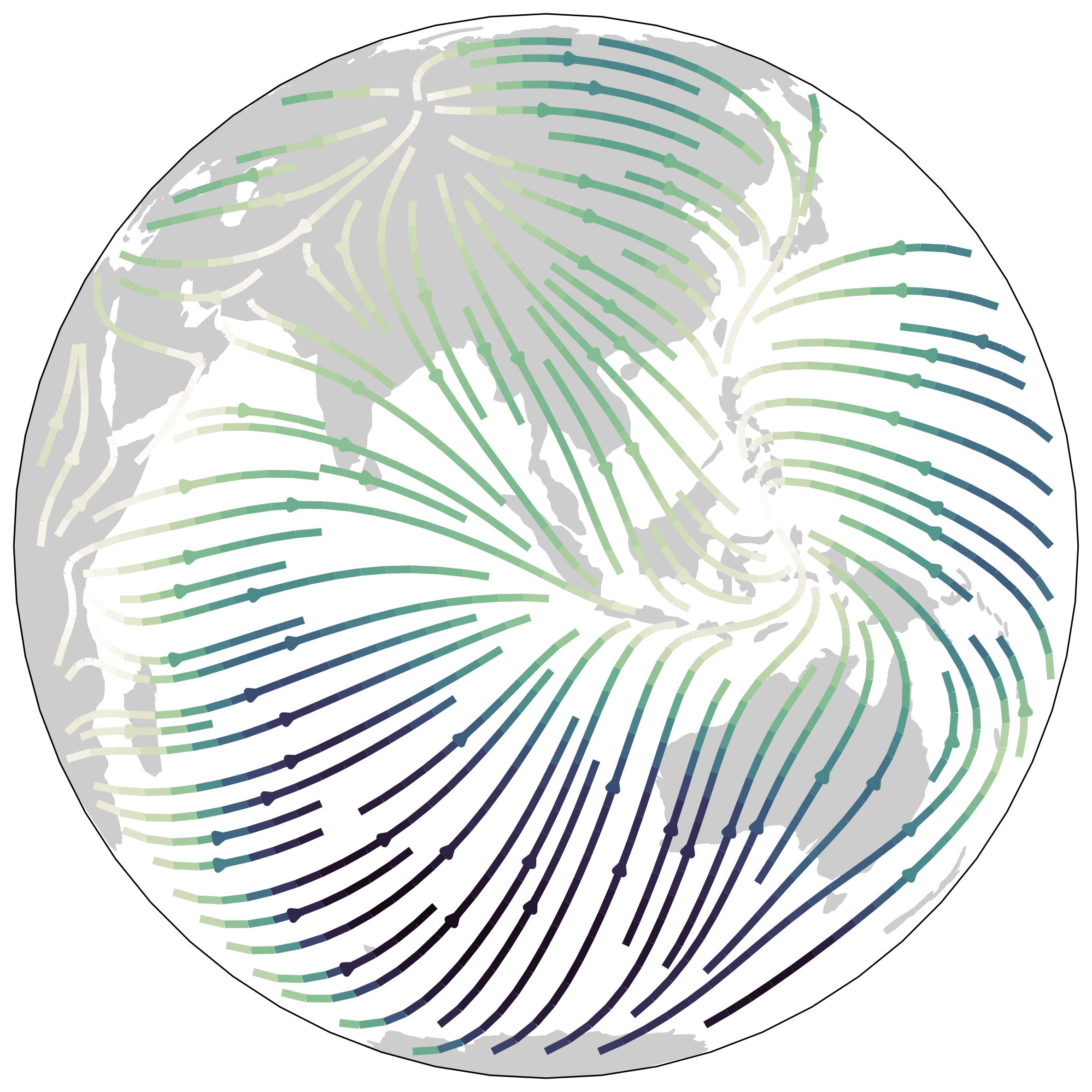}
  \caption{
    Trajectories generated on the sphere to model volcano eruptions.
    Note that these converge to the known \emph{Ring of Fire}.
  }
  \label{fig:sphere_integral_path}
  \vspace{-0.8em}
\end{wrapfigure}
Learning well-specified probabilistic models is at the heart of many problems in
machine learning and statistics. Much focus has therefore been placed on
developing methods for modelling and inferring expressive probability distributions.
\Acrlongpl{NF} \citep{rezende2016Variational} have shown great promise for
this task as they provide a general and extensible framework for modelling highly complex and multimodal distributions~\citep{papamakarios2019Normalizing}.

An orthogonal but equally important aspect of well-specified models is to
correctly characterize the geometry which describes the proximity of data points.
Riemannian manifolds provide a general framework for this purpose and are a natural
approach to model tasks in many scientific fields ranging from earth and climate
science to biology and computer vision.
For instance, storm trajectories may be modelled as paths on the sphere
\citep{karpatne2017Machine}, the shape of proteins can be parametrized using
tori \citep{hamelryck2006Sampling}, cell developmental processes can be described
through paths in hyperbolic space \citep{klimovskaia2019Poincare}, and human
actions can be recognized in video using matrix
manifolds \citep{lui2012Advances}. If appropriately chosen, manifold-informed
methods can lead to improved sample complexity and generalization, improved fit
in the low parameter regime, and guide inference methods to interpretable
models. They can also be understood as a geometric prior that encodes a
practitioner's assumption about the data and imposes an inductive bias.

However, conventional normalizing flows are not readily applicable to such manifold-valued
data since their implicit Euclidean assumption makes them unaware of the
underlying geometry or borders of the manifold. As a result they would yield
distributions having some or all of their mass lying outside the manifold,
rendering them ill-suited or even misspecified so that central concepts like the
reverse \gls{KL} divergence would not even be defined.

In this work, we propose a principled way to combine both of these aspects and
parametrize flexible probability distributions on Riemannian manifolds.
Specifically, we introduce \emph{Riemmanian continuous normalizing flows} in
which flows are defined via vector fields on manifolds and computed as the
solution to the associated \gls{ODE}  (see \cref{fig:sphere_integral_path} for an illustration). Intuitively, our method operates by first
parametrizing a vector field on the manifold with a neural network, then sampling
particles from a base distribution, and finally
approximating their flow along the vector field using a numerical solver. Both the
neural network and the solver are aware of the underlying geometry which ensures
that the flow is always located on the manifold -- yielding a \emph{Riemannian} method.

This approach allows us to combine multiple important advantages: One major
challenge of normalizing flows lies in designing transformations that enable
efficient sampling and density computation. By basing our approach on
\emph{\glspl{CNF}} \citep{chen2019Neural,grathwohl2018FFJORD,salman2018Deep}
we avoid strong structural constraints to be imposed on
the flow, as is the case for most \emph{discrete} \acrlongpl{NF}.
Such unconstrained \emph{free-form} flows have empirically been shown to be highly expressive \citep{chen2020Residual,grathwohl2018FFJORD}.
Moreover, \emph{projected} methods require a differentiable mapping from a Euclidean space to the manifold, yet such a function cannot be bijective, which in turn leads to numerical challenges.
By taking a Riemannian approach, our method is more
versatile since it does not rely on an ad-hoc projection map and simultaneously reduces numerical artefacts that interfere with training.
To the best of our knowledge, our method is the first to combine these
properties as existing methods for normalizing flows on manifolds are either
discrete~\citep{bose2020Latent, rezende2020Normalizing},
projected~\citep{gemici2016Normalizing,falorsi2019Reparameterizing,bose2020Latent} or manifold-specific \citep{sei2011Jacobian, bose2020Latent, rezende2020Normalizing}.

We empirically demonstrate the advantages of our method on constant curvature
manifolds -- i.e., the
Poincar\'e disk and the sphere -- and show the benefits of the proposed approach
compared to non-Riemannian and projected methods for maximum likelihood
estimation and reverse \gls{KL} minimization. We also apply our method to density
estimation on earth-sciences data (e.g., locations of earthquakes, floods and
wildfires) and show that it yields better generalization performance and
faster convergence.
%


\begin{table}[b]
\centering
\caption{Summary of d-dimensional continuous constant (sectional) curvature manifolds.}
\label{table:geometry}
\resizebox{\textwidth}{!}{
\begin{tabular}{lcllllll}
\toprule
  Geometry  & & Model & Curvature & Coordinates  &  $\sqrt{\det g}=d\vol/d\leb_{\R^d}$  & Compact  \\
  \cmidrule(r){1-2}\cmidrule(lr){3-3}\cmidrule(lr){4-4}\cmidrule(lr){5-5}\cmidrule(l){6-6}\cmidrule(l){7-7}
  \bf Euclidean & $\R^d$ & Real vector space & $K = 0$ & Cartesian $\z$ & $1$ & No \\
  \bf Hyperbolic & $\B_K^d$ & Poincar\'e ball  & $K < 0$ & Cartesian $\z$ & $\left(2~ / ~1 + K\left\|\z\right\|^2 \right)^d $ & No \\
 \bf Elliptic & $\S_K^d$ & Hypersphere & $K > 0$ & n-spherical $\bm{\varphi}$ &  $K^{-\frac{d-1}{2}} \prod^{d-2}_{i=1} {\sin(\varphi_{i})^{d-i-1}}$ & Yes \\
\bottomrule
\end{tabular}
}
\end{table}

\section{Continuous Normalizing Flows on Riemannian Manifolds} \label{sec:model}
\Acrlongpl{NF} operate by pushing a simple base distribution through a
series of parametrized invertible maps, referred as the \emph{flow}. This can
yield a highly complex and multimodal distribution which is
typically assumed to live in a Euclidean vector space.
%
Here, we propose a principled approach to extend \acrlongpl{NF} to
manifold-valued data, i.e.\ \emph{\acrlongpl{RCNF}} (\acrshortpl{RCNF}).
Following \glspl{CNF} \citep{chen2019Neural,grathwohl2018FFJORD,salman2018Deep} we define
manifold flows as the solutions to \glspl{ODE}. The high-level idea is to
parametrize flows through the time-evolution of manifold-valued particles $\z$
-- in particular via their velocity $\dot{\z}(t) = \f_{\theta}(\z(t), t)$ where $\f_{\theta}$ denotes a
\emph{vector field}. Particles are first sampled from a simple base
distribution, and then their evolution is integrated by a manifold-aware
numerical solver, yielding a new complex multimodal distribution of the
particles.
This \emph{Riemannian and continuous} approach has the advantages of allowing
almost \emph{free-form} neural networks and of not requiring any mapping from a
Euclidean space which would potentially lead to numerical challenges.

For practical purposes, we focus our theoretical and experimental discussion on
\emph{constant curvature manifolds} (see \cref{table:geometry}). In addition to
being widely used in the literature \citep{nickel2017Poincar,davidson2018Hyperspherical,mardia2000Directional,hasnat2017MisesFisher}, these manifolds are convenient to work with
since most related geometrical quantities are available in closed-form. However,
our proposed approach is generic and could be used on a broad class of manifolds
such as product and matrix manifolds like tori and Grassmanians. For a brief
overview of relevant concepts in Riemannian geometry please see
\cref{sec:riem_review} or \cite{lee2003Introduction} for a more thorough
introduction.

In the following, we develop the key components which allow us to define continuous
normalizing flows that are aware of the underlying Riemannian geometry:
flow, likelihood, and vector field. 

\paragraph{Vector flows}
Flows in conventional \acrlongpl{NF} are defined as smooth mappings
$\phi: \R^{d} \to \R^{d}$ which transform a base distribution $z \sim P_{0}$
into a complex distribution $P_{\theta}$. For \acrlongpl{NF} to be well-behaved
and convenient to work with, the flow is required to be \emph{bijective} and
\emph{differentiable} which introduces significant structural constraints on
$\phi$. Continuous normalizing flows overcome this issue by defining the flow $\phi: \R^{d} \times \R \to \R^{d}$ generated by an
\acrlong{ODE}, allowing for unrestricted neural network
architectures. Here we show how \emph{vector fields} can be used to define
similar flows $\phi: \M \times \R \to \M$ on general \emph{Riemannian manifolds}.

Consider the temporal evolution of a particle $\z(t)$ lying on a d-dimensional
manifold $\M$, whose velocity is given by a \emph{vector field} $\f_{\theta}(\z(t), t)$.
Intuitively, $\f_{\theta}(\z(t), t)$ indicates the direction and speed along
which the particle is moving on the \emph{manifold's surface}. Classic examples for
such vector fields include weathercocks giving wind direction and compasses
pointing toward the magnetic north pole of the earth.
Formally, let $\T_{\z}\M$ denote the \emph{tangent space} at $\z$ and
$\T\M = \cap_{\z \in \M} ~\T_{\z}\M$ the associated \emph{tangent bundle}.
Furthermore, let $\f_{\theta}: \M \times \R \mapsto \T\M$ denote a vector field
on $\M$. The particle's time-evolution according to $\f_{\theta}$ is then given by the following \gls{ODE}
\begin{align} \label{eq:ode}
	\frac{d\z(t)}{dt} = \f_{\theta}(\z(t), t).
\end{align}
To transform a base distribution using this vector field, we are then interested
in a particle's position after time $t$.
When starting at an initial position $\z(0)=\z_0$,
the \emph{flow} operator $\phi: \M \times \R \mapsto \M$ gives the particle's position at any time $t$ as
$z(t) = \phi(\z_0, t)$.
Leveraging the fundamental theorem of flows \citep{lee2003Introduction}, we can show
that under mild conditions, this flow
is \emph{bijective} and \emph{differentiable}.
We write $C^1$ for the set of differentiable functions whose derivative are continuous.
\begin{proposition}[Vector flows]\label{prop:vff}
Let $\M$ be a \emph{smooth} complete manifold.
Furthermore, let $\f_{\theta}$ be a $C^1$-bounded time-dependent \emph{vector field}.
Then there exists a \emph{global flow} $\phi: \M \times \R \mapsto \M$
such that for each $t \in \R$,
the map $\phi(\cdot, t): \M \mapsto \M$ is a $C^1$-diffeomorphism (i.e.\ $C^1$ bijection with $C^1$ inverse).
\end{proposition}
\begin{proof}
See \cref{sec:global_flow} for a detailed derivation.
\end{proof}
Note that scaling the vector field as $f^\alpha_\theta \triangleq \alpha f_\theta$ results in a time-scaled flow $\phi^\alpha(\z, t) = \phi(\z, \alpha t)$.
The integration duration $t$ is therefore arbitrary.
Without loss of generality we set $t=1$ and write $\phi \triangleq \phi(\cdot, 1)$.
Concerning the evaluation of the flow $\phi$, it generally does no accept a closed-form solution and thus requires to be approximated numerically. To this extent we rely on an explicit and adaptive \gls{RK} integrator of order 4 \citep{dormand1980family}.
However, standard integrators used in \glspl{CNF} generally do not preserve manifold constraints \citep{hairer2006Geometric} .
To overcome this issue we rely on a \emph{projective} solver \citep{hairer2011Solving}.
This solver works by conveniently solving the \gls{ODE} in the ambient Cartesian coordinates and projecting each step onto the manifold.
Projections onto $\S^d$ are computationally cheap since they amount to $l^2$ norm divisions.
No projection is required for the Poincar\'e ball.
\paragraph{Likelihood}
Having a flow at hand, we are now interested in evaluating the likelihood of our
\emph{pushforward} model $P_\theta=\phi_\sharp P_0$. Here, the \emph{pushforward} operator $\sharp$ indicates that
one obtains samples $\z \sim \phi_\sharp P_0$ as $\z=\phi(\z_0)$ with
$z_0 \sim P_0$.
For this purpose, we derive in the following the change in density in terms of the geometry of
the manifold and show how to efficiently estimate the likelihood.

\subparagraph{Change in density}
In \acrlongpl{NF}, we can compute the likelihood of a sample via the change
in density from the base distribution to the pushforward.
Applying the chain rule we get
\begin{align} \label{eq:discrete_change_of_var}
	\log p_{\theta}(\z) - \log p_0(\z_0)
	= \log  \left| \det \frac{\partial \phi^{-1}(z)} {\partial \z} \right|
	= - \log \left| \det \frac{\partial \phi(z_0)}{\partial \z} \right|.
\end{align}
In general, computing the Jacobian's determinant of the flow is challenging since it requires $d$ reverse-mode automatic differentiations to obtain the full Jacobian matrix, and $\mathcal{O}(d^3)$ operations to compute its determinant.
\Glspl{CNF} side step direct computation of the determinant by leveraging the
time-continuity of the flow and re-expressing \cref{eq:discrete_change_of_var}
as the integral of the \emph{instantaneous change} in log density
$\int_0^t \frac{\partial \log p_{\theta}(\z(t))}{\partial t} ~dt$.
However, standard CNFs
make an implicit Euclidean assumption to compute this quantity which
is violated for general Riemannian manifolds. To overcome this issue we express
the instantaneous change in log-density in terms of the \emph{Riemannian metric}.
In particular, let  $G(\z)$ denote the \emph{matrix representation of the Riemannian metric} for a given manifold $\M$, then $G(\z)$ endows tangent spaces $\T_{\z} \M$ with an inner product.
For instance in the Poincar\'e ball $\B^d$, it holds that
$G(\z)= (2~/~1 + K\left\|\z\right\|^2) ~I_d$, while in Euclidean space $\R^d$ we have $G(\z) = I_{d}$, where $I_{d}$ denotes the identity matrix.
Using the Liouville equation, we can then show that the instantaneous change in
variable is defined as follows.
\begin{proposition}[Instantaneous change of variables] \label{prop:instant_change}
Let $\z(t)$ be a continuous manifold-valued random variable given in local coordinates, which is described by the \gls{ODE} from \cref{eq:ode}
 with probability density $p_{\theta}(\z(t))$. 
The change in log-probability then also follows a differential equation
 given by
\begin{align}
\frac{\partial \log p_{\theta}(\z(t))}{\partial t} &= - \diver (\f_{\theta}(\z(t), t))
= - {|G(\z(t))|}^{-\frac{1}{2}}  ~\tr \left( \frac{\partial \sqrt{|G(\z(t))|} \f_{\theta}(\z(t), t)}{\partial \z} \right) \label{eq:continuous_change_of_var} \\
&= - \tr \left( \frac{\partial \f_{\theta}(\z(t), t)}{\partial \z} \right)
- {|G(\z(t))|}^{-\frac{1}{2}} ~\Bigl< \f_\theta(\z(t), t), \frac{\partial}{\partial \z} \sqrt{|G(\z(t))|}~ \Bigr>. \label{eq:continuous_change_of_var2}
\end{align}
\end{proposition}
\begin{proof}
For a detailed derivation of \cref{eq:continuous_change_of_var} see \cref{sec:proof}.
\end{proof}
Note that in the Euclidean setting $\sqrt{|G(\z)|}=1$ thus the second term of \cref{eq:continuous_change_of_var2} vanishes and we recover the formula from \cite{grathwohl2018FFJORD,chen2019Neural}.
\subparagraph{Estimating the divergence}
Even though the determinant of \cref{eq:discrete_change_of_var} has been replaced in \cref{eq:continuous_change_of_var} by a trace operator with lower computational complexity, we still need to compute the full Jacobian matrix of $f_\theta$.
Similarly to \cite{grathwohl2018FFJORD,salman2018Deep}, we
make use of Hutchinson’s trace estimator to
compute the Jacobian efficiently. In particular,
\citet{hutchinson1990stochastic} showed that
$
\tr(A) = \mathbb{E}_{p(\bm{\epsilon})}[\bm{\epsilon}^{\intercal} A \bm{\epsilon}]
$
with $p(\bm{\epsilon})$ being a $d$-dimensional random vector such that
$\mathbb{E}[\bm{\epsilon}] = 0$ and $\text{Cov}(\bm{\epsilon}) = I_d$.
Leveraging this trace estimator to approximate the divergence in \cref{eq:continuous_change_of_var} yields
\begin{align} \label{eq:stochastic_div}
  \diver (\f_{\theta}(\z(t), t)) = {|G(\z(t))|}^{-\frac{1}{2}} 
  \mathbb{E}_{p(\bm{\epsilon})} \left[ \bm{\epsilon}^\intercal \frac{\partial \sqrt{|G(\z(t))|} \f_{\theta}(\z(t), t)}{\partial \z} \bm{\epsilon} \right].
\end{align}
We note that the variance of this estimator can potentially be high since it scales with the inverse of the determinant term $\sqrt{|G(\z(t))|}$ (see \cref{sec:regularization_frobenius}).
By integrating \cref{eq:continuous_change_of_var} over time with the stochastic divergence estimator from \cref{eq:stochastic_div}, we get the following total change in log-density between the manifold-valued random variables $\z$ and $\z_0$
\begin{align} \label{eq:log_like}
\hspace{-0.2em}
\log\left( \frac{p_{\theta}(\z)}{p_{0}(\z_0)} \right)
=  - \hspace{-0.2em} \int_{0}^{1} \hspace{-0.2em} \diver (\f_{\theta}(\z(t), t)) ~dt
 = - \mathbb{E}_{p(\bm{\epsilon})} \left[  \int_{0}^{1} \hspace{-0.2em} {|G(\z(t))|}^{-\frac{1}{2}}   \bm{\epsilon}^\intercal \frac{\partial \sqrt{|G(\z(t))|} \f_{\theta}(\z(t), t)}{\partial \z} \bm{\epsilon}~dt \right].
\end{align}
It can be seen that \cref{eq:log_like} accounts again for the underlying
geometry through the metric $G(\z(t))$. \Cref{table:geometry} lists
closed-form solutions of its determinant for constant curvature manifolds.
Furthermore, the vector-Jacobian product can be
computed through backward auto-differentiation with linear complexity, avoiding
the quadratic cost of computing the full Jacobian matrix. Additionally, the integral is
approximated via the discretization of the flow returned by the solver.

\subparagraph{Choice of base distribution $P_0$} The closer the initial base
distribution is to the target distribution, the easier the learning task
should be. However, it is challenging in practice to incorporate such prior
knowledge. We consequently use a uniform distribution on $\S^d$ since it is
the most "uncertain" distribution.
For the Poincar\'e ball $\B^d$, we rely on a standard wrapped Gaussian distribution $\mathcal{N}^{\text{W}}$ \citep{nagano2019Wrapped,mathieu2019Continuous} because it is convenient to work with. 
\paragraph{Vector field} \label{sec:vector_field}
Finally, we discuss the form of the \emph{vector field}
${\f_{\theta}: \M \times \R \to \T\M}$ which generates the flow $\phi$ used to
pushforward samples. We parametrize $\f_{\theta}$ via a feed-forward neural network
which takes as input manifold-valued particles, and outputs their velocities. 
The architecture of the vector field has a direct impact on the expressiveness of the distribution and is thus crucially important.
In order to take into account these geometrical properties we make use of specific input and output layers that we describe below.
The rest of the architecture is based on a \acrlong{MLP}.
\subparagraph{Input layer}
To inform the neural network about the geometry of the manifold $\M$, we use as first layer a \emph{geodesic distance layer}
\citep{ganea2018Hyperbolic,mathieu2019Continuous} which generalizes \emph{linear
  layers} to manifolds, and can be seen as computing distances to \emph{decision
  boundaries} on $\M$.
These boundaries are parametrized by \emph{geodesic hyperplanes} $H_{\w}$, and the associated \emph{neurons} $h_{\w}(\z) \propto d_{\M}(\z,H_{\w})$,
with $d_{\M}$ being the geodesic distance.
Horizontally stacking several of these \emph{neurons} makes a \emph{geodesic distance layer}.
We refer to \cref{sec:geodesic_layer} for more details.
\subparagraph{Output layer}
To constrain the neural net to $\T\M$, we output vectors in $\R^{d+1}$ when ${\M=\S^d}$, before projecting them to the tangent space
i.e.\ $\f_{\theta}(\z) = \text{proj}_{\T_{\z}\M} ~\texttt{neural\_net}(\z)$.
This is not necessary in $\B^d$ since the ambient space is of equal dimension.
Yet, velocities scale as
${\|\f_{\theta}(\z)\|_{\z} = {|G(\z)|}^{1/2} ~\|\f_{\theta}(\z)\|_{2}}$, hence
we scale the $\texttt{neural\_net}$ by ${|G(\z)|}^{-1/2}$
s.t.\ $\|\f_{\theta}(\z)\|_{\z} = \|\texttt{neural\_net}(\z)\|_{2}$.

\subparagraph{Regularity}
For the flow to be bijective, the vector field $\f_{\theta}$ is required to be $C^1$ and bounded (cf \cref{prop:vff}).
The boundness and smoothness conditions can be satisfied by relying on bounded smooth non-linearities in $\f_{\theta}$ such as tanh, along with bounded weight and bias at the last layer.
\paragraph{Training}
In density estimation and inference tasks, one aims to learn a model $P_\theta$ with parameters $\theta$
by minimising a \emph{divergence} $\mathcal{L}(\theta) = D(P_{\mathcal{D}} ~||~ P_{\theta})$ w.r.t. a target distribution $P_{\mathcal{D}}$.
In our case, the parameters $\theta$ refer to the parameters of the vector field $\f_{\theta}$.
We minimize the loss $\mathcal{L}(\theta)$ using first-order stochastic optimization, which requires \acrlong{MC} estimates of loss gradients $\nabla_\theta ~\mathcal{L}(\theta)$.
We back-propagate gradients through the explicit solver with $O(1/h)$ memory cost, $h$ being the step size.
When the loss $\mathcal{L}(\theta)$ is expressed as an expectation over the model $P_{\theta}$, as in the \emph{reverse} \gls{KL} divergence, we rely on the \emph{reparametrization trick} \citep{kingma2014AutoEncoding,rezende2014Stochastic}.
In our experiments we will consider both the negative log-likelihood and reverse \gls{KL} objectives
\begin{align}
  \mathcal{L}^{\text{Like}}(\theta) = - \mathbb{E}_{\z \sim P_\mathcal{D}} \left[ \log p_\theta(\z) \right]
  \ \text{and} \
  \mathcal{L}^{\text{KL}}(\theta) = \KL{P_\theta}{P_\mathcal{D}} = \mathbb{E}_{\z \sim P_\theta} \left[ \log p_\theta(\z) - \log p_\mathcal{D}(\z) \right].
\end{align}
Additionally, regularization terms can be added in the hope of improving
training and generalization. See \cref{sec:regularization} for a
discussion and connections to the dynamical formulation of optimal transport.



\section{Related work} \label{sec:related_work} \label{sec:manifold_methods}
Here we discuss previous work that introduced \acrlongpl{NF} on manifolds.
For clarity we split these into \emph{projected} vs \emph{Riemannian} methods which we describe below.
\paragraph{Projected methods} \label{sec:projected_methods}
These methods consist in parametrizing a normalizing flow on $\R^d$ and then pushing-forward the resulting distribution along an invertible map $\psi:\R^d \rightarrow \M$.
Yet, the existence of such an \emph{invertible} map is equivalent to $\M$ being homeomorphic to $\R^d$ (e.g.\ being "flat"), hence limiting  the scope of that approach.
Moreover there is no principled way to choose such a map, and different choices lead to different numerical or computational challenges which we discuss below.
\subparagraph{Exponential map}
The first generic \emph{projected} map that comes to mind in this setting is the exponential map $\exp_{\bm{\mu}}: T_{\bm{\mu}}\M \cong \R^d \to \M$, which parameterizes geodesics starting from $\bm{\mu}$ with velocity $\v \in T_{\bm{\mu}}\M$.
This leads to so called \emph{wrapped} distributions $P^{\text{W}}_{\theta} = \exp_{\bm{\mu}\sharp}  P$, with $P$ a probability measure on $\R^d$.
This approach has been taken by \cite{falorsi2019Reparameterizing} to parametrize probability distributions on Lie groups. 
Yet, in compact manifolds -- such as spheres or the $\text{SO}(3)$ group -- 
computing the density of \emph{wrapped} distributions requires an infinite summation, which in practice needs to be truncated.
This is not the case however on hyperbolic spaces (like the Poincar\'e ball)
since the exponential map is bijective on these manifolds. 
This approach has been proposed in \cite{bose2020Latent} where they extend Real-NVP \citep{dinh2017Density} to the hyperboloid model of hyperbolic geometry.
In addition to this \emph{wrapped} Real-NVP, they also introduced a hybrid coupling model which is empirically shown to be more expressive.
We note however that the exponential map is believed to be "badly behaved" away from the origin \citep{dooley1993Harmonic,al-mohy2010New}.
%
\subparagraph{Stereographic map}
Alternatively to the exponential map, \citet{gemici2016Normalizing} proposed to parametrize probability distributions on $\S^d$ via the \emph{stereographic projection} defined as
${\rho(\z) = {\z_{2:d}}~/~ ({1 + ~z_1})}$
with \emph{projection point} $- \{\bm{\mu}_0\} = (-1, 0, \dots, 0)$.
\citeauthor{gemici2016Normalizing} then push a probability measure $P$ defined on $\R^d$ along the inverse of the \emph{stereographic} map $\rho$, yielding $P^{\text{S}}_{\theta} = \rho^{-1}_\sharp P$.
However, the stereographic map $\rho$ is not injective, and projects $-\bm{\mu}_0$ to $\infty$.
This implies that spherical points close to the projection point $- \{\bm{\mu}_0\}$ are mapped far away from the origin of the plane.
Modelling probability distributions with mass close to $\{- \bm{\mu}_0\}$ may consequently be numerically challenging since the norm of the Euclidean flow would explode.
Similarly, \cite{rezende2020Normalizing} introduced flows on hyperspheres and tori by using the inverse tangent function.
Although this method is empirically shown to perform well, it similarly suffers from numerical instabilities near singularity points.
\paragraph{Riemannian methods}
In contrast to \emph{projected} methods which rely on mapping the manifold to a Euclidean space, \emph{Riemannian} methods do not.
As a consequence they side-step any artefact or numerical instability arising from the manifold's projection.
Early work \citep{sei2011Jacobian} proposed transformations along geodesics on the hypersphere by evaluating the exponential map at the gradient of a scalar manifold function.
Recently, \cite{rezende2020Normalizing} introduced ad-hoc \emph{discrete} \emph{Riemannian} flows for hyperspheres and tori based on M\"{o}bius transformations and spherical splines.
We contribute to this line of work by introducing \emph{continuous} flows on
general Riemannian manifolds.
In contrast to \emph{discrete} flows \citep[e.g.][]{bose2020Latent,rezende2020Normalizing}, \emph{time-continuous} flows as ours alleviate strong structural constraints on the flow by implicitly parametrizing it as the solution to an \gls{ODE} \citep{grathwohl2018FFJORD}.
Additionally, recent and concurrent work \citep{lou2020Neural,falorsi2020Neural} proposed to extend neural ODEs to smooth manifolds.
%


\begin{figure}[t]
\centering
\begin{minipage}{0.52\textwidth}
  \centering
  \begin{subfigure}{.25\textwidth}
    \includegraphics[width=\linewidth, trim={6em 0 23em 0},clip]{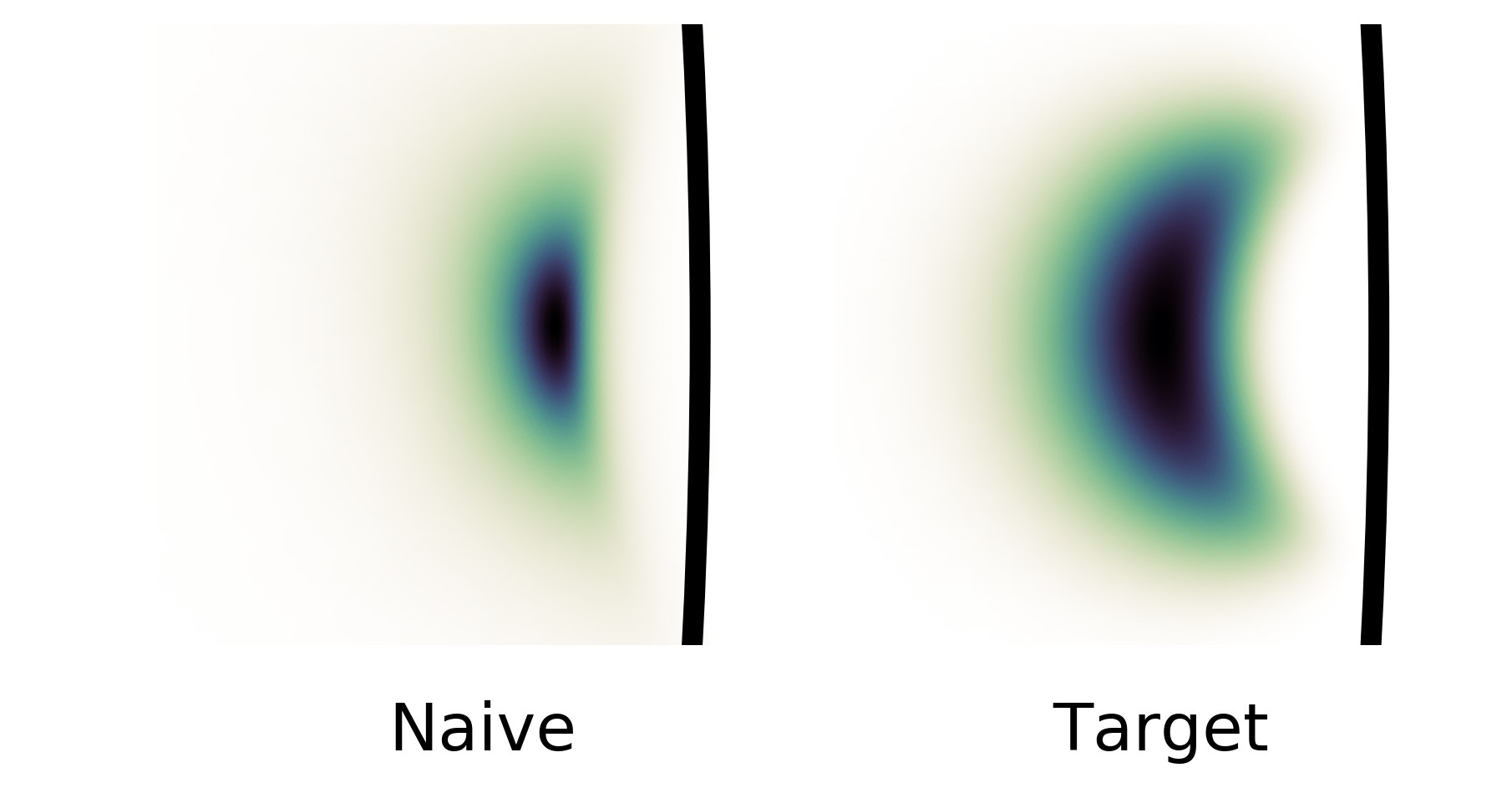}
  \end{subfigure}\hfil
  \begin{subfigure}{.25\textwidth}
    \includegraphics[width=\linewidth, trim={6em 0 23em 0},clip]{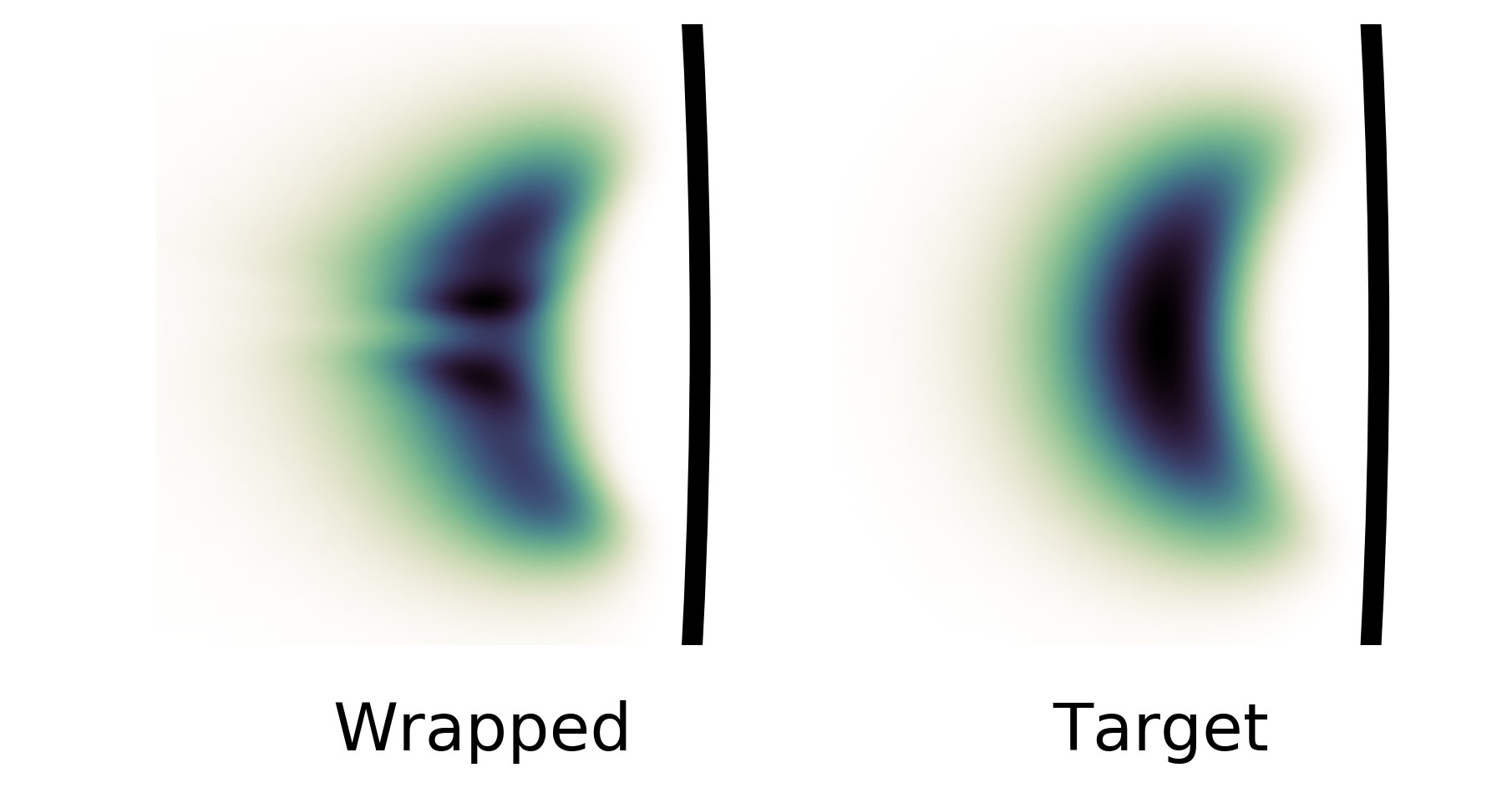}
  \end{subfigure}\hfil
  \begin{subfigure}{.25\textwidth}
    \includegraphics[width=\linewidth, trim={6em 0 23em 0},clip]{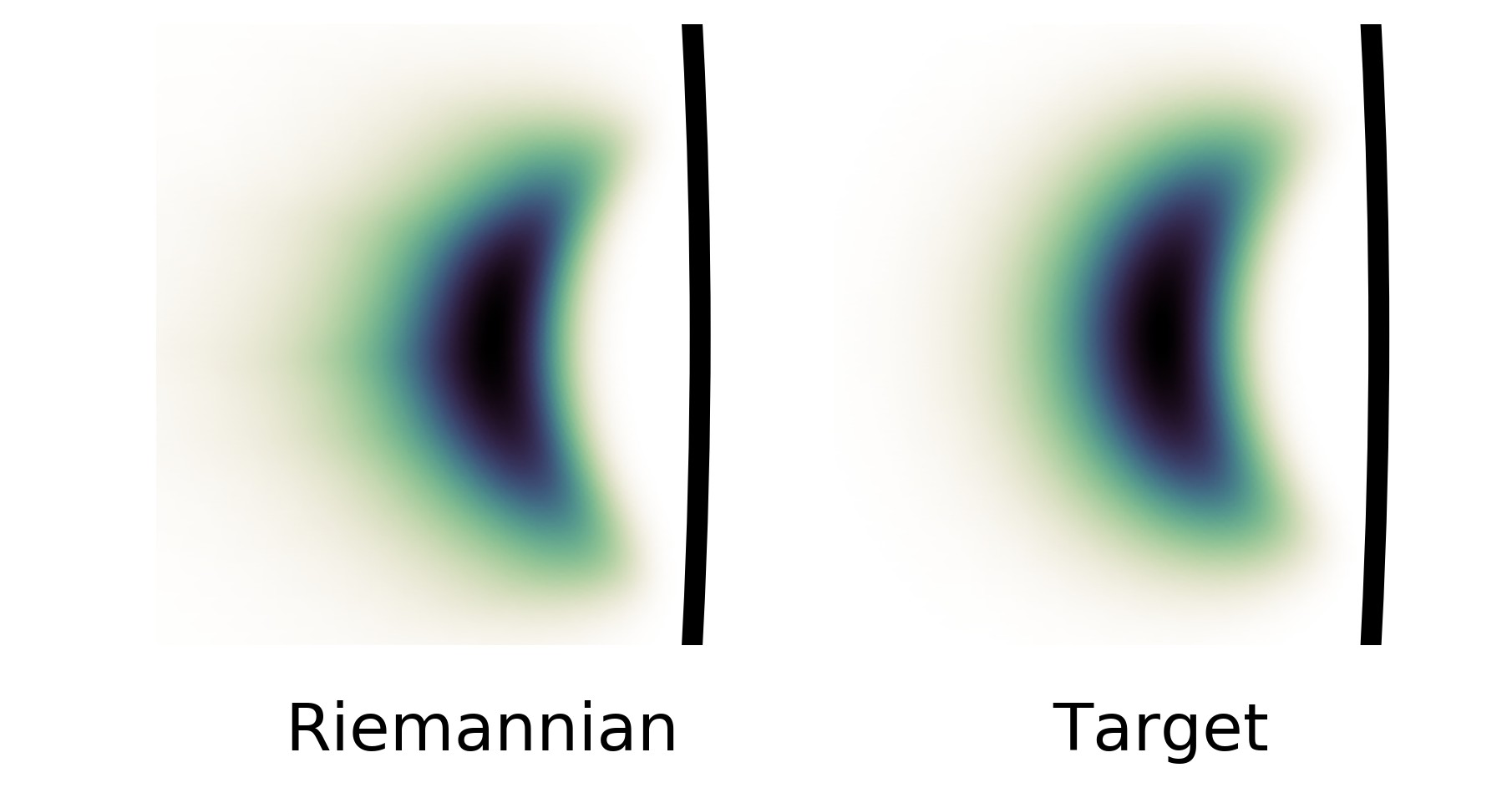}
  \end{subfigure}\hfil
  \begin{subfigure}{.25\textwidth}
    \includegraphics[width=\linewidth, trim={26em 0 3em 0},clip]{{disk/densities/density_like_2_riem}.jpg}
  \end{subfigure}
    \caption{
    Probability densities on $\B^2$. Models have been trained by \acrlong{maxlike} to fit $\mathcal{N}^{\text{W}}(\exp_{\bm{0}}(2~\partial x), \Sigma)$. The black semi-circle indicates the disk's border.
    The best run out of twelve trainings is shown for each model.
    }
    \label{fig:poicare_disc_density}
  \end{minipage}\hfill
  \begin{minipage}{0.45\textwidth}
    \centering
    \includegraphics[width=0.95\textwidth]{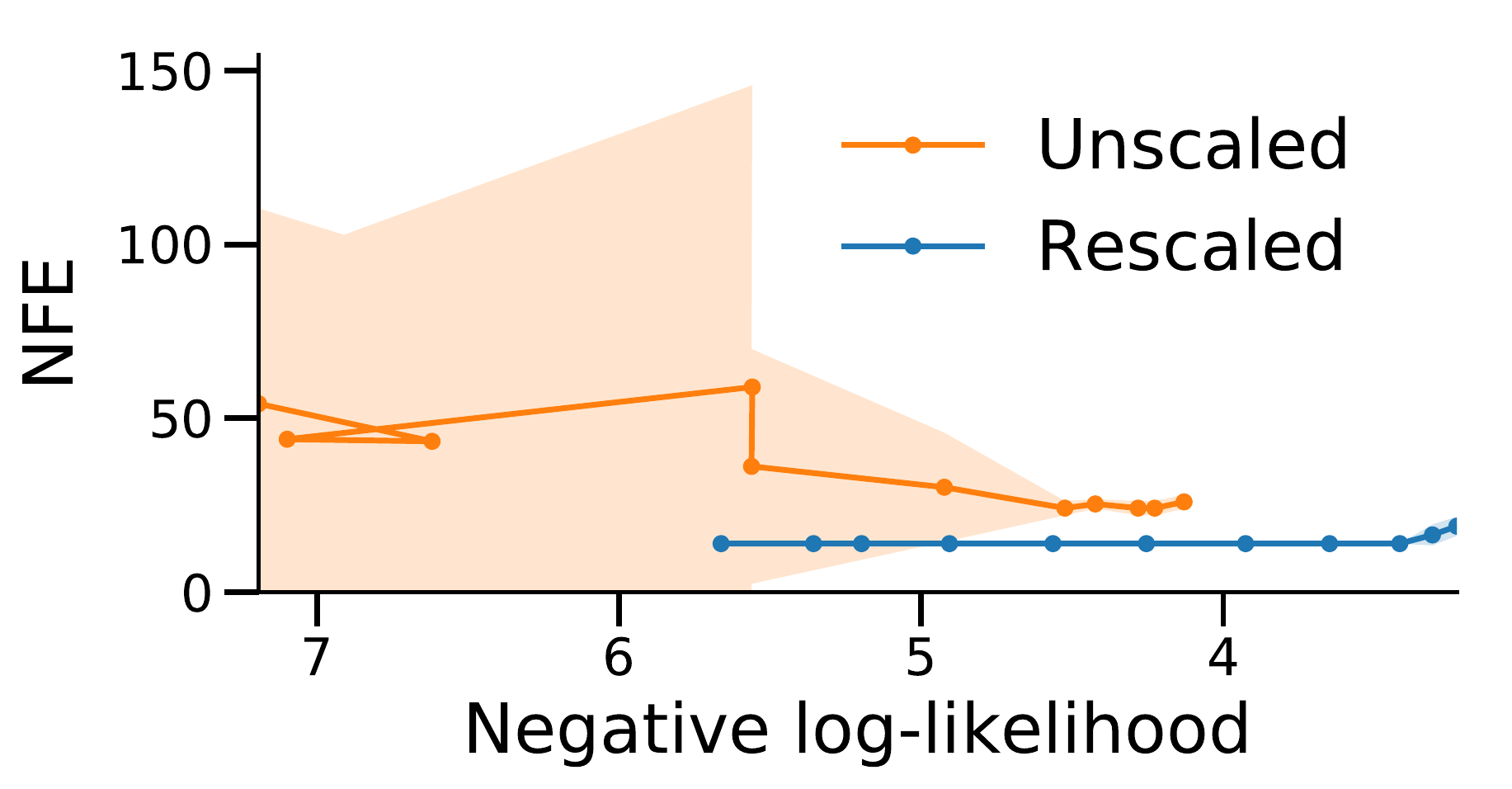}
    \caption{
      Ablation study of the vector field architecture for the \emph{Riemannian} model.
      Models have been trained to fit a $\mathcal{N}^{\text{W}}(\exp_{\bm{0}}(\partial x), \Sigma)$.
      }
      \label{fig:disk_ablation_scaling}
  \end{minipage}
\end{figure}
\begin{figure}[t]
\centering
\begin{subfigure}{1.\textwidth}
  \includegraphics[width=\linewidth]{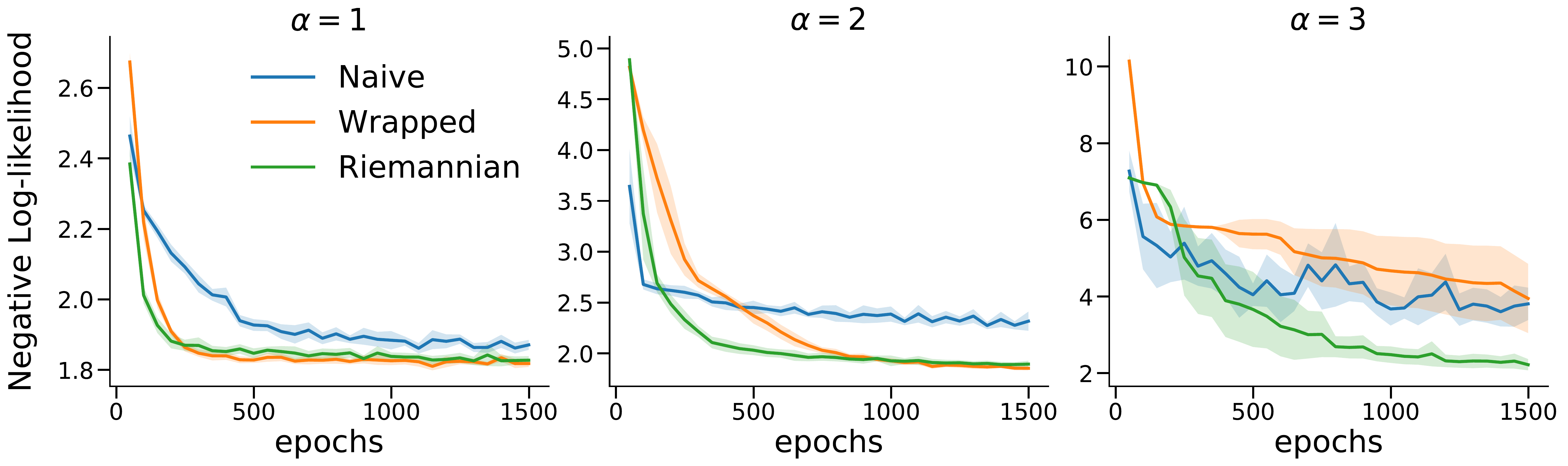}
\end{subfigure}
  \caption{
   Negative Log-likelihood of \glspl{CNF} trained to fit a $\mathcal{N}^{\text{W}}(\exp_{\bm{0}}(\alpha ~\partial x), \Sigma)$ target on $\B^2$.
  }
  \label{fig:poicare_convergence}
\end{figure}

\section{Experimental results} \label{sec:experiments}
We evaluate the empirical performance of the above-mentioned models on
hyperbolic and spherical geometry.
We will first discuss experiments on two synthetic datasets where we highlight specific pathologies of the naive and projected methods via unimodal distributions at the point (or the limit) of the pathology.
This removes additional modelling artefacts that would be introduced through more complex distributions and allows to demonstrate advantages of our approach on the respective manifolds.
We further show that these advantages also translate to
substantial gains on highly multi-modal real world datasets.

For all \emph{projected} models (e.g.\ stereographic and wrapped cf
\cref{sec:projected_methods}), the vector field's architecture is chosen to be
a \acrlong{MLP} as in \cite{grathwohl2018FFJORD}, whilst the architecture described in \cref{sec:vector_field} is used for our \emph{Riemannian} (continuous normalizing flow) model.
For fair comparisons, we also parametrize \emph{projected} models with a \gls{CNF}.
Also, all models are chosen to have approximately the same number of parameters.
All models were implemented in PyTorch~\citep{paszke2017automatic} and trained by stochastic optimization with Adam~\citep{kingma2015Adam}.
All $95\%$ confidence intervals are computed over 12 runs. Please refer to
\cref{sec:exp_details} for full experimental details.
%
\paragraph{Hyperbolic geometry and limits of conventional and wrapped methods}
First, we aim to show that conventional \acrlongpl{NF} are ill-suited for modelling target manifold distributions.
These are blind to the geometry, so we expect them to behave poorly when the target is located where the manifold 
behaves most differently from a Euclidean space.
We refer to such models as \emph{naive} and discuss their properties in more detail in \cref{sec:ambient_distribution}.
Second, we wish to inspect the behaviour of \emph{wrapped} models (see \cref{sec:projected_methods}) when the target is away from the exponential map origin.

To this extent we parametrize a wrapped Gaussian target distribution
$\mathcal{N}^{\text{W}}(\exp_{\bm{0}}(\alpha~\partial x), \Sigma)
= \exp_{\bm{\mu} \sharp} \mathcal{N}(\alpha~\partial x, \Sigma)$
defined on the Poincar\'e disk $\B^2$ \citep{nagano2019Wrapped,mathieu2019Continuous}.
The scalar parameter $\alpha$ allows us to locate the target closer or further away from the origin of the disk.
We put three \glspl{CNF} models on the benchmark; our \emph{Riemannian} (from \cref{sec:model}), a conventional \emph{naive} and a \emph{wrapped} model.
The base distribution $P_0$ is a standard Gaussian for the \emph{naive} and \emph{wrapped} models, and a standard wrapped Gaussian for the \emph{Riemannian} model.
Models are trained by \acrlong{maxlike} until convergence.
Throughout training, the Hutchinson's estimator is used to approximate the divergence as in \cref{eq:stochastic_div}.
It can be seen from \cref{fig:poicare_convergence} that the \emph{Riemannian}
model indeed outperforms the \emph{naive} and \emph{wrapped} models as we
increase the values of $\alpha$ -- i.e., the closer we move to the boundary of
the disk.
\cref{fig:poicare_disc_density} shows that qualitatively the \emph{naive} and \emph{wrapped} models seem to indeed fail to properly fit the target when it is located far from the origin.
Additionally, we assess the architectural choice of the vector field used in our \emph{Riemannian} model.
In particular, we conduct an ablation study on the rescaling of the output layer, by training for $10$ iterations a \emph{rescaled} and an \emph{unscaled} version of our model.
\cref{fig:disk_ablation_scaling} shows that the \gls{NFE} tends to be large and sometimes even dramatically diverges when the vector field's output is \emph{unscaled}.
In addition to increasing the computational cost, this in turns appears to worsen the convergence's speed of the model.
This further illustrates the benefits of our vector field parameterization.

\paragraph{Spherical geometry and limits of the stereographic projection model}
\begin{figure}[t]
  \begin{minipage}[b]{.01\textwidth}
    \hspace{.1em}
  \end{minipage}
  \hfill
  \begin{minipage}[b]{0.46\textwidth}
   \vspace{.5em}
    \centering
    \includegraphics[width=\linewidth]{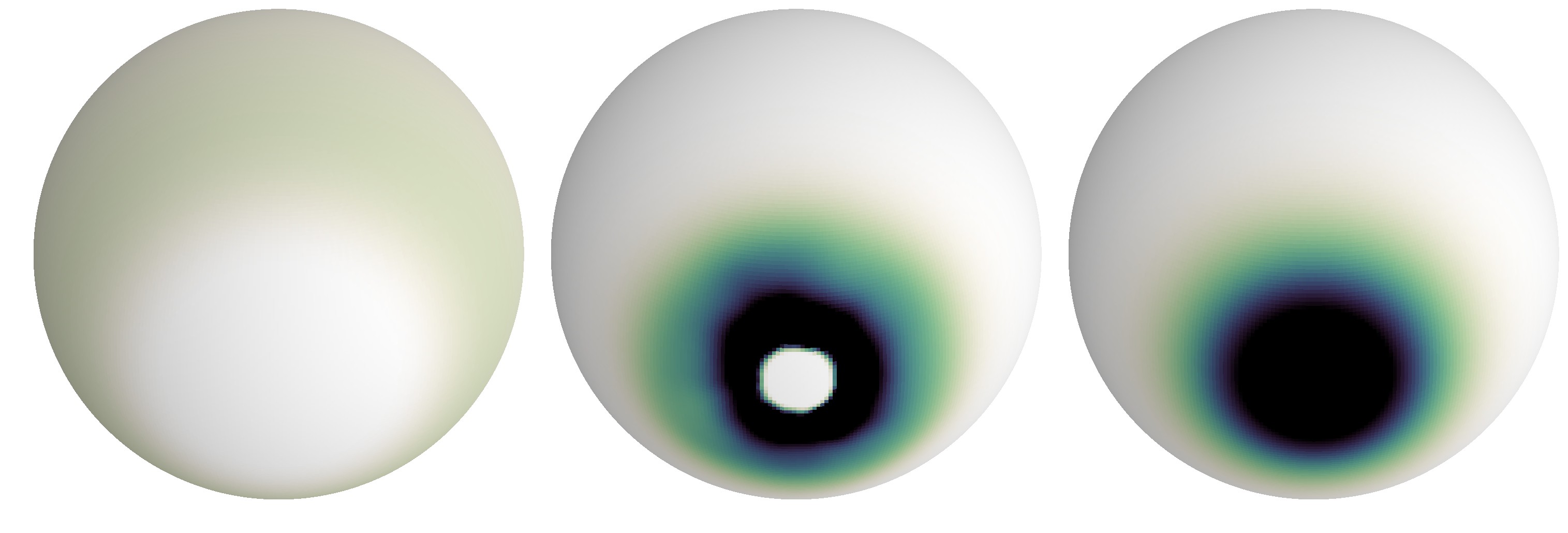}
    \put(-193, 6){\rotatebox{90}{Stereographic}} \\
    \includegraphics[width=\linewidth]{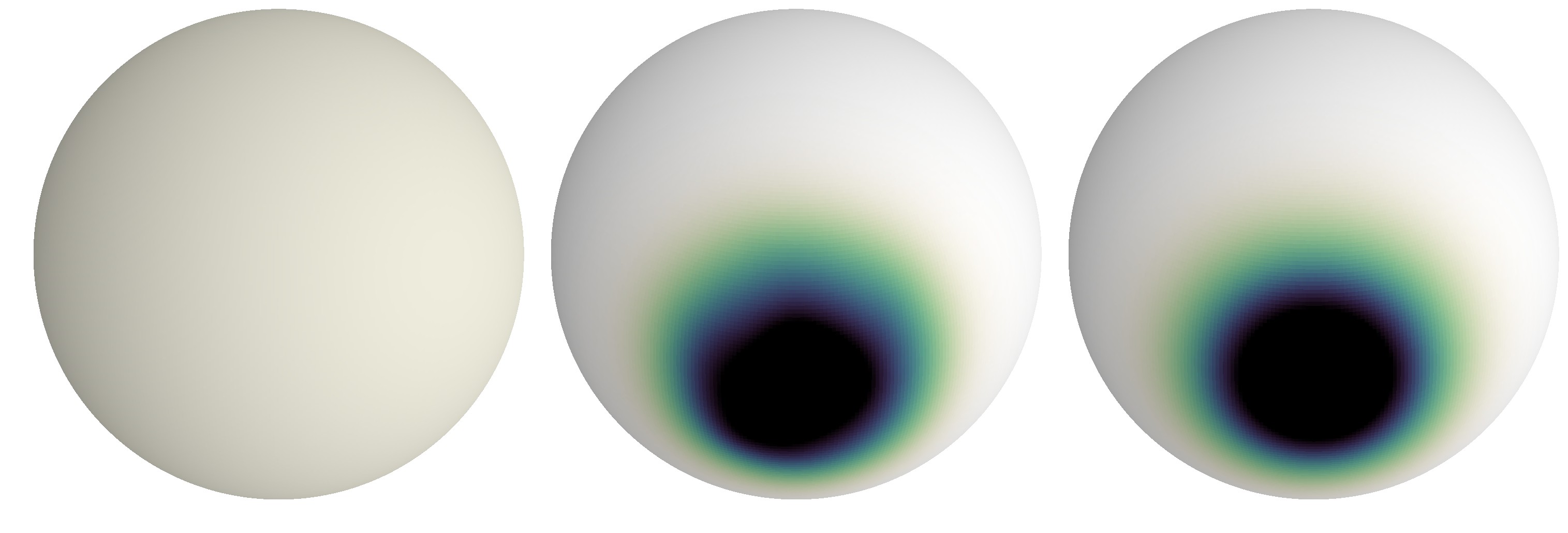}
    \put(-193, 6){\rotatebox{90}{Riemannian}}
    \put(-165,-10){Base $P_0$}
    \put(-105,-10){Model $P_{\theta}$}
    \put(-50,-10){Target $P_{\mathcal{D}}$}
    \captionof{figure}{
      Probability distributions on $\S^2$.
      Models trained to fit a $\vMF(\bm{\mu} = -\bm{\mu}_0, \kappa=10)$.
    }
    \label{fig:vmf}
  \end{minipage}
  \hfill
  \begin{minipage}[b]{0.49\textwidth}
    \centering
    \begin{tabular}{cccccc}
\toprule
                          & model &     \bf Stereographic &           \bf Riemannian \\
Loss & ${\kappa}$ &                       &                          \\
\midrule
\multirow{3}{*}{$\mathcal{L}^{\text{Like}}$} & 100 &  ${63.60}_{\pm 3.56}$ &  $\bm{-1.78}_{\pm 0.01}$ \\
                          & 50  &  ${32.68}_{\pm 3.15}$ &  $\bm{-1.09}_{\pm 0.01}$ \\
                          & 10  &   ${6.45}_{\pm 2.42}$ &   $\bm{0.52}_{\pm 0.01}$ \\
\cline{1-4}
\multirow{3}{*}{$\mathcal{L}^{\text{KL}}$} & 100 &   ${1.56}_{\pm 0.34}$ &   $\bm{0.04}_{\pm 0.02}$ \\
                          & 50  &   ${0.68}_{\pm 0.16}$ &   $\bm{0.03}_{\pm 0.02}$ \\
                          & 10  &   ${0.12}_{\pm 0.01}$ &   $\bm{0.01}_{\pm 0.00}$ \\
\bottomrule
\end{tabular}

      \captionof{table}{
      Performance of continuous flows on $\S^2$ with $\vMF(\bm{\mu}=-\bm{\mu}_0, \kappa)$ targets (the smaller the better).
      When models perfectly fit the target $P_{\mathcal{D}}$, then $\mathcal{L}^{\text{Like}} = \mathbb{H}[P_{\mathcal{D}}]$ which decreases with $\kappa$, explaining $\mathcal{L}^{\text{Like}}$'s results for the \emph{Riemannian} model.
      }
      \label{table:vmf}
    \end{minipage}
  \end{figure}
  
Next, we evaluate the ability of our model and the stereographic projection
model from
\cref{sec:projected_methods} to approximate distributions on the sphere
which are located around the projection point $- \bm{\mu}_0$.
We empirically assess this phenomenon by choosing the target distribution to be a Von-Mises Fisher \citep{downs1972Orientational} distribution $\vMF(\bm{\mu}, \kappa)$ located at $\mu= - \bm{\mu}_0$, and with concentration $\kappa$ (which decreases with the variance).
Along with the \emph{stereographic} projection method, we also consider our \emph{Riemannian} model from \cref{sec:model}.
We neither included the \emph{naive} model since it is misspecified here
(leading to an undefined reverse KL divergence), nor the \emph{wrapped} model as computing its density requires an infinite summation (see \cref{sec:projected_methods}).
The base distribution $P_0$ is chosen
to be a standard Gaussian on $\R^2$ for the \emph{stereographic} model and a uniform distribution on $\S^2$ for the \emph{Riemannian} model.
Models are trained by computing the exact divergence.
The performance of these two models are quantitatively assessed on both the negative log-likelihood and reverse \gls{KL} criteria.

\cref{fig:vmf} shows densities of the target distribution along with the base and learned distributions.
We observe that the stereographic model fails to push mass close enough to the singularity point $-\bm{\mu}_0$, as opposed to the Riemannian model which perfectly fits the target.
\cref{table:vmf} shows the negative log-likelihood and reverse \gls{KL} losses of both models when varying the concentration parameter $\kappa$ of the $\vMF$ target.
The larger the concentration $\kappa$ is, the closer to the singularity point $-\bm{\mu}_0$ the target's mass gets.
We observe that the \emph{Riemannian} model outperforms the \emph{stereographic} one to fit the target for both objectives, although this performance gap shrinks as the concentration gets smaller.
Also, we believe that the gap in performance is particularly large for the log-likelihood objective because it heavily penalizes models that fail to cover the support of the target.
When the $\vMF$ target is located away from the singularity point, we noted that both models were performing similarly well.
\paragraph{Density estimation of spherical data}
Finally, we aim to measure the expressiveness and modelling capabilities of our
method on real world datasets. To this extent, we gathered four earth location
datasets, representing respectively volcano eruptions \citep{data_volcano},
earthquakes \citep{data_earthquake}, floods \citep{data_flood} and wild fires
\citep{data_fire}. We approximate the earth's surface (and thus also
these data points) as a perfect sphere.
Along our \emph{Riemannian} \gls{CNF}, we also assess the fitting capacity of a mixture of von Mises-Fisher (vMF) distributions and a \emph{stereographic} projected \gls{CNF}.
The locations of the vMF components are learned via stochastic Riemannian optimization \citep{bonnabel2013Stochastic,becigneul2019Riemannian}.
The learning rate and number of components are selected by hyperparameter grid search.
In our experiments, we split datasets randomly into training and testing datasets, 
and fit the models by \acrlong{maxlike} estimation on the training dataset.
\Gls{CNF} models are trained by computing the exact divergence.

We observe from \cref{table:max_like_sphere} that for all datasets, the
\emph{Riemannian} model outperforms its \emph{stereographic} counterpart and the mixture of vMF distributions by a large margin.
It can also be seen from the learning curves that the \emph{Riemannian} model converges faster.
\cref{fig:earth} shows the learned spherical distributions along with the
training and testing datasets. We note that qualitatively the
\emph{stereographic} distribution is generally more diffuse than its
\emph{Riemannian} counterpart. It also appears to allocate some of its mass
outside the target support, and to cover less of the data points. Additional
figures are shown in \cref{sec:add_figures}.
\begin{table}[t]
  \centering
  \caption{
  Negative test log-likelihood of continuous normalizing flows on $\S^2$ datasets.
  }
  \label{table:max_like_sphere}
\begin{tabular}{lccccc}
\toprule
& Volcano & Earthquake & Flood & Fire\\ \midrule 
{\bf Mixture vMF} \hfill {\color{blue}$\blacksquare$}  & ${-0.31}_{\pm 0.07}$ & ${0.59}_{\pm 0.01}$ & ${1.09}_{\pm 0.01}$ & ${-0.23}_{\pm 0.02}$\\
\cmidrule(r){1-1}\cmidrule(lr){2-5} 
{\bf Stereographic} \hfill {\color{orange}$\blacksquare$}  & ${-0.64}_{\pm 0.20}$ & ${0.43}_{\pm 0.04}$ & ${0.99}_{\pm 0.04}$ & ${-0.40}_{\pm 0.06}$\\
\cmidrule(r){1-1}\cmidrule(lr){2-5} 
{\bf Riemannian} \hfill {\color{green}$\blacksquare$}  & ${-0.97}_{\pm 0.15}$ & $\bm{0.19}_{\pm 0.04}$ & $\bm{0.90}_{\pm 0.03}$ & $\bm{-0.66}_{\pm 0.05}$\\
\cmidrule(r){1-1}\cmidrule(lr){2-5} 
\raisebox{{3em}}{{Learning curves}} & \hspace{-0.5em} \includegraphics[width=.175\linewidth] {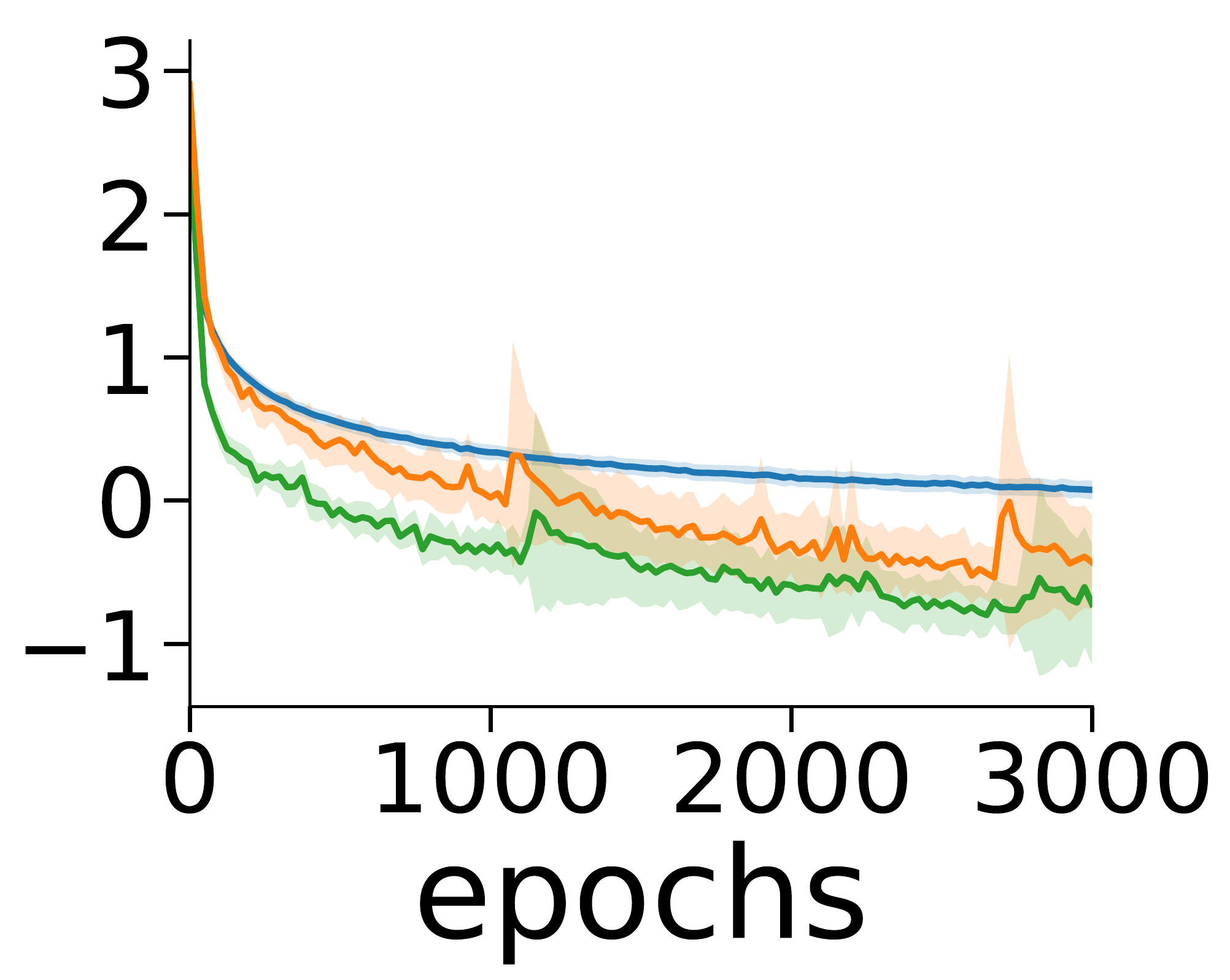} & \hspace{-0.5em} \includegraphics[width=.175\linewidth] {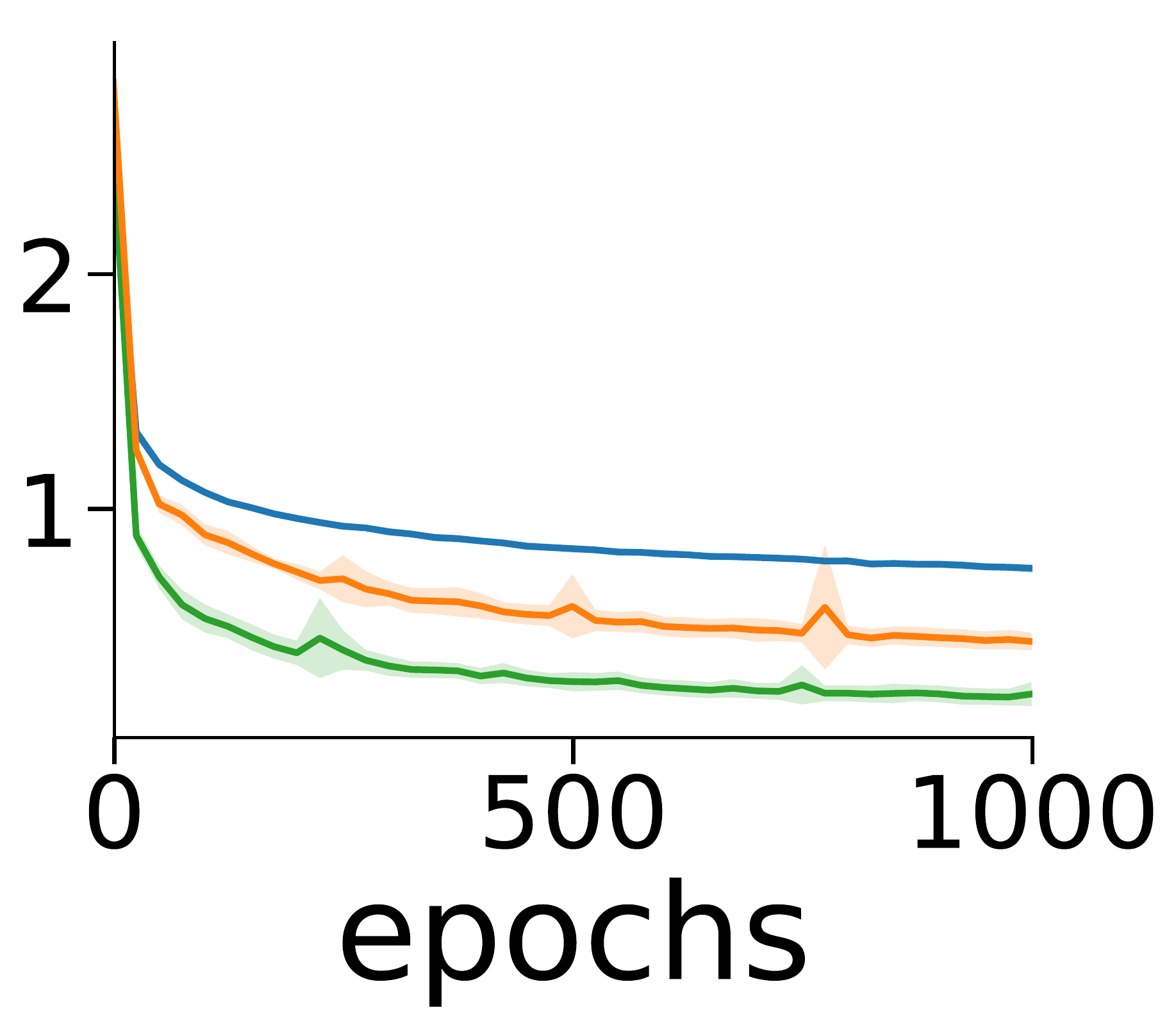} & \hspace{-0.5em} \includegraphics[width=.175\linewidth] {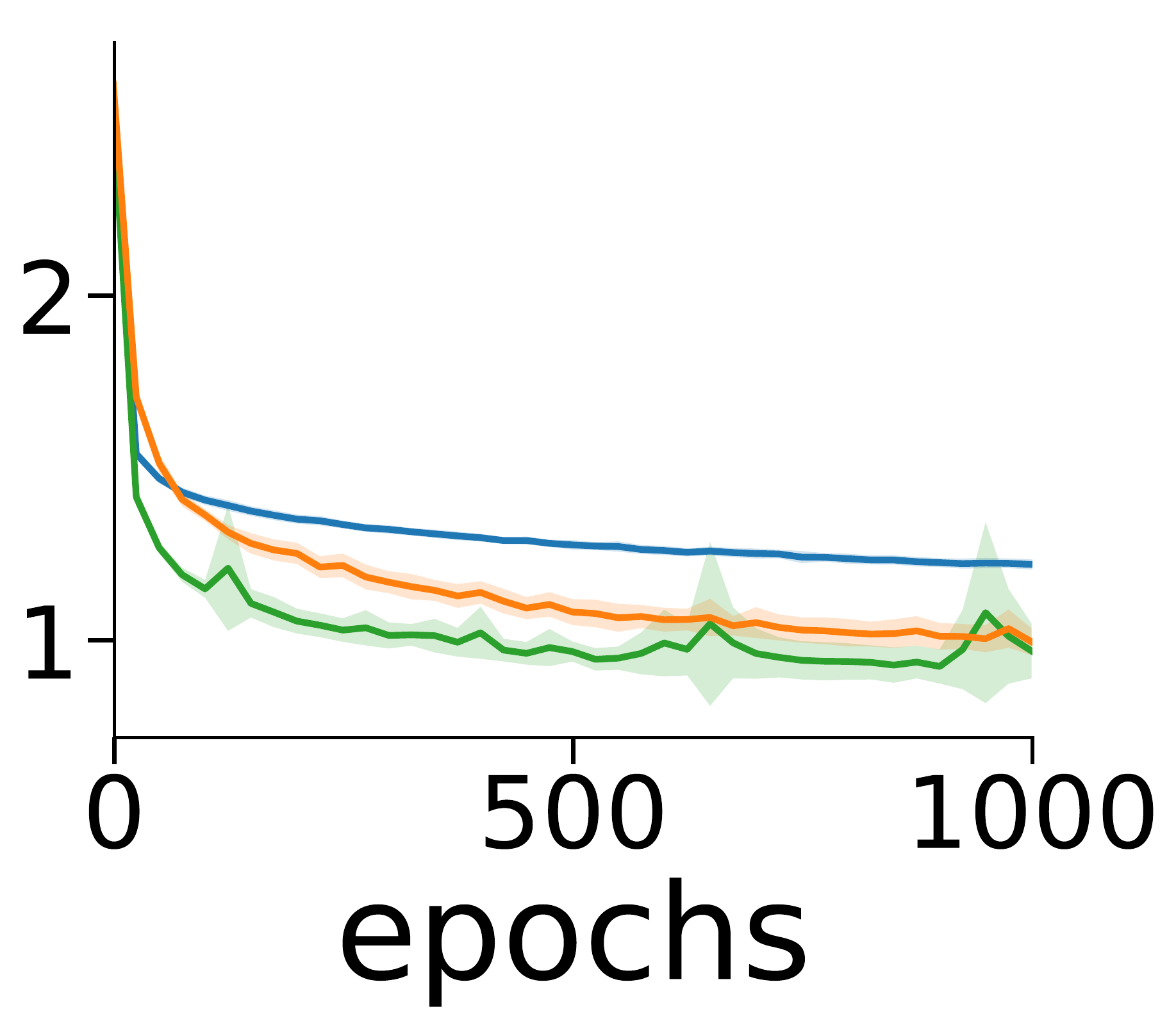} & \hspace{-0.5em} \includegraphics[width=.175\linewidth] {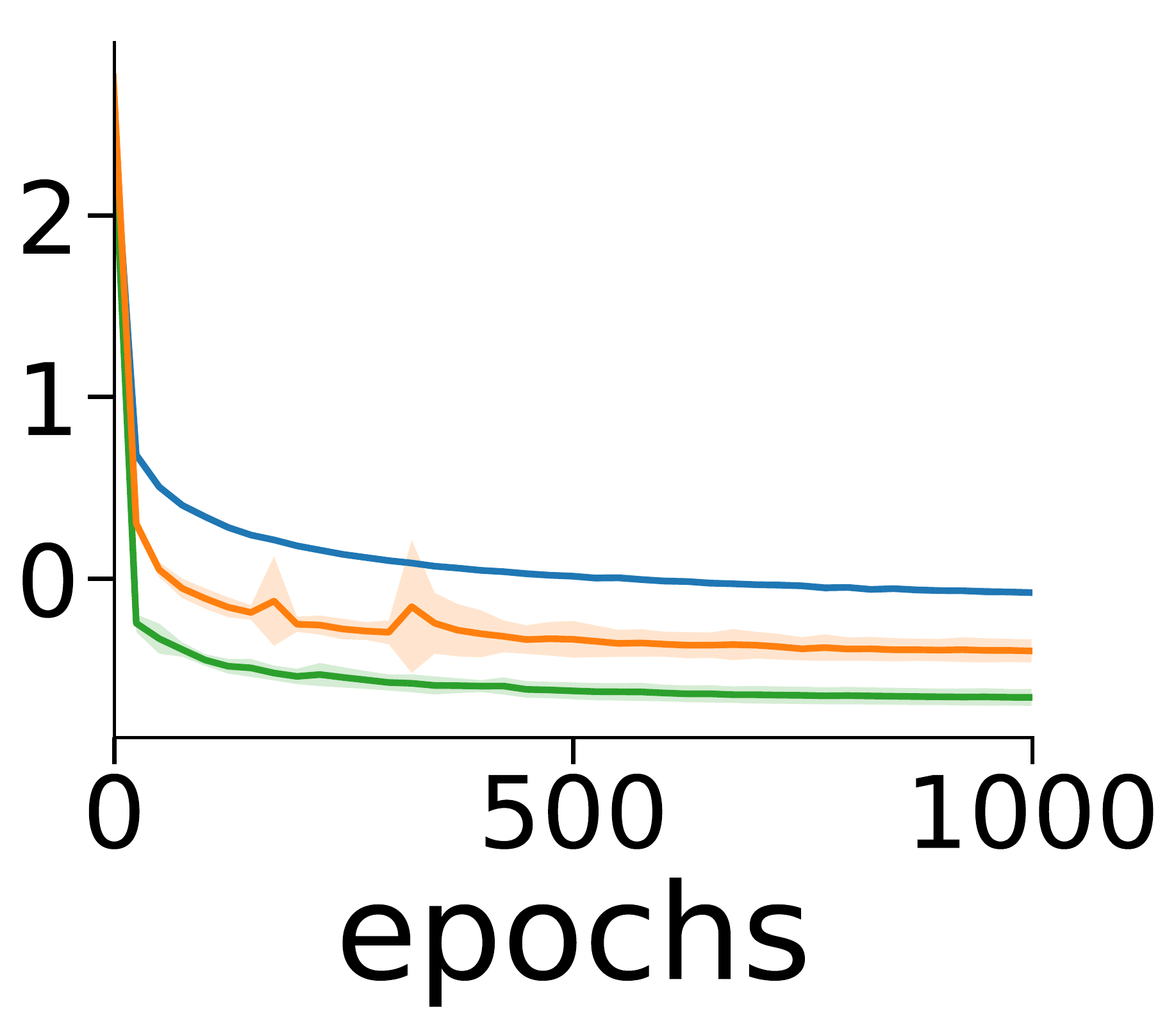}\\
\cmidrule(r){1-1}\cmidrule(lr){2-5} 
Data size & 829 & 6124 & 4877 & 12810\\
\bottomrule
\end{tabular}
\end{table}
\begin{figure}[t]
  \centering
\begin{subfigure}{.33\textwidth}
  \includegraphics[width=\linewidth]{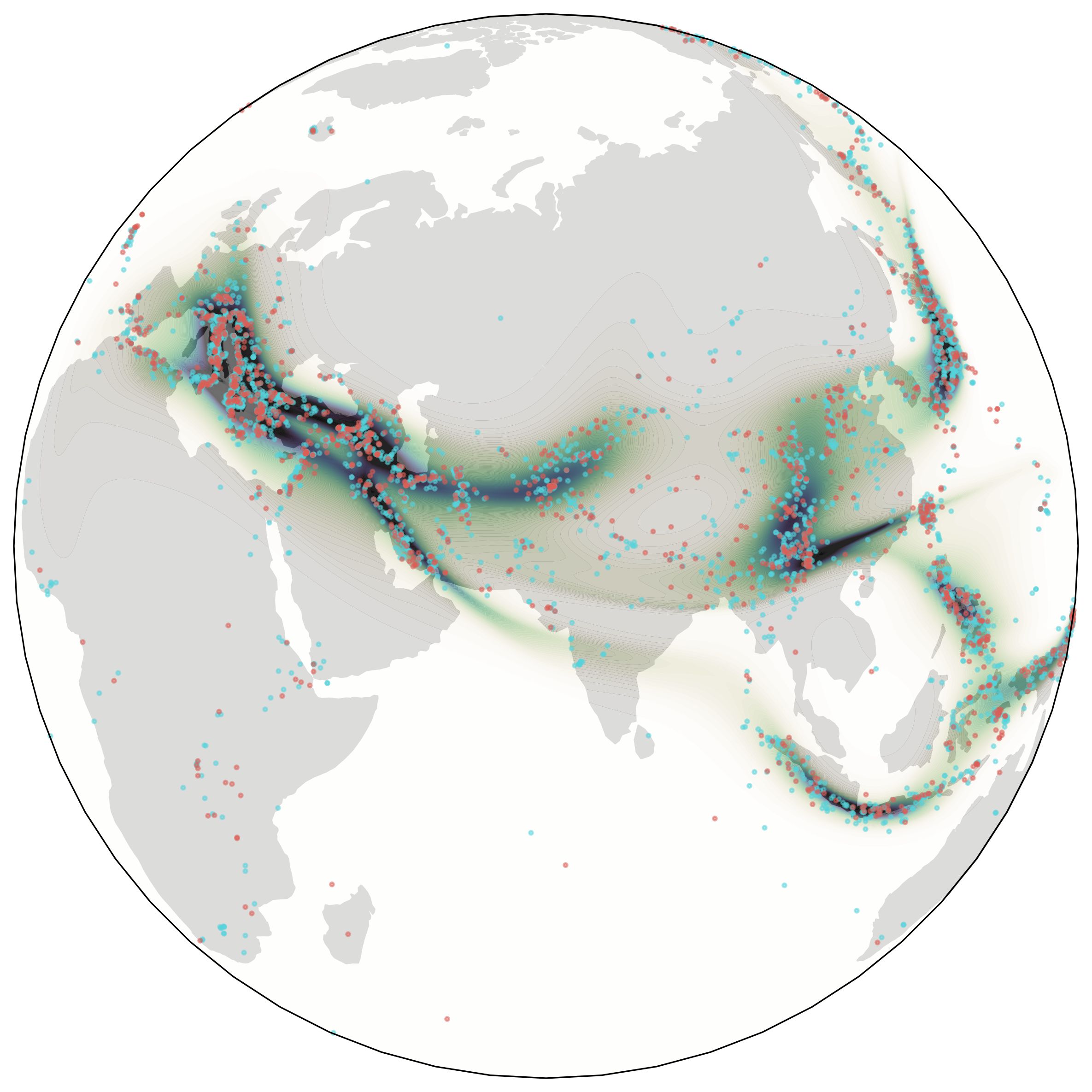}
  \put(-150,40){\rotatebox{90}{Stereographic}}
\end{subfigure}\hfil
\begin{subfigure}{.33\textwidth}
  \includegraphics[width=\linewidth]{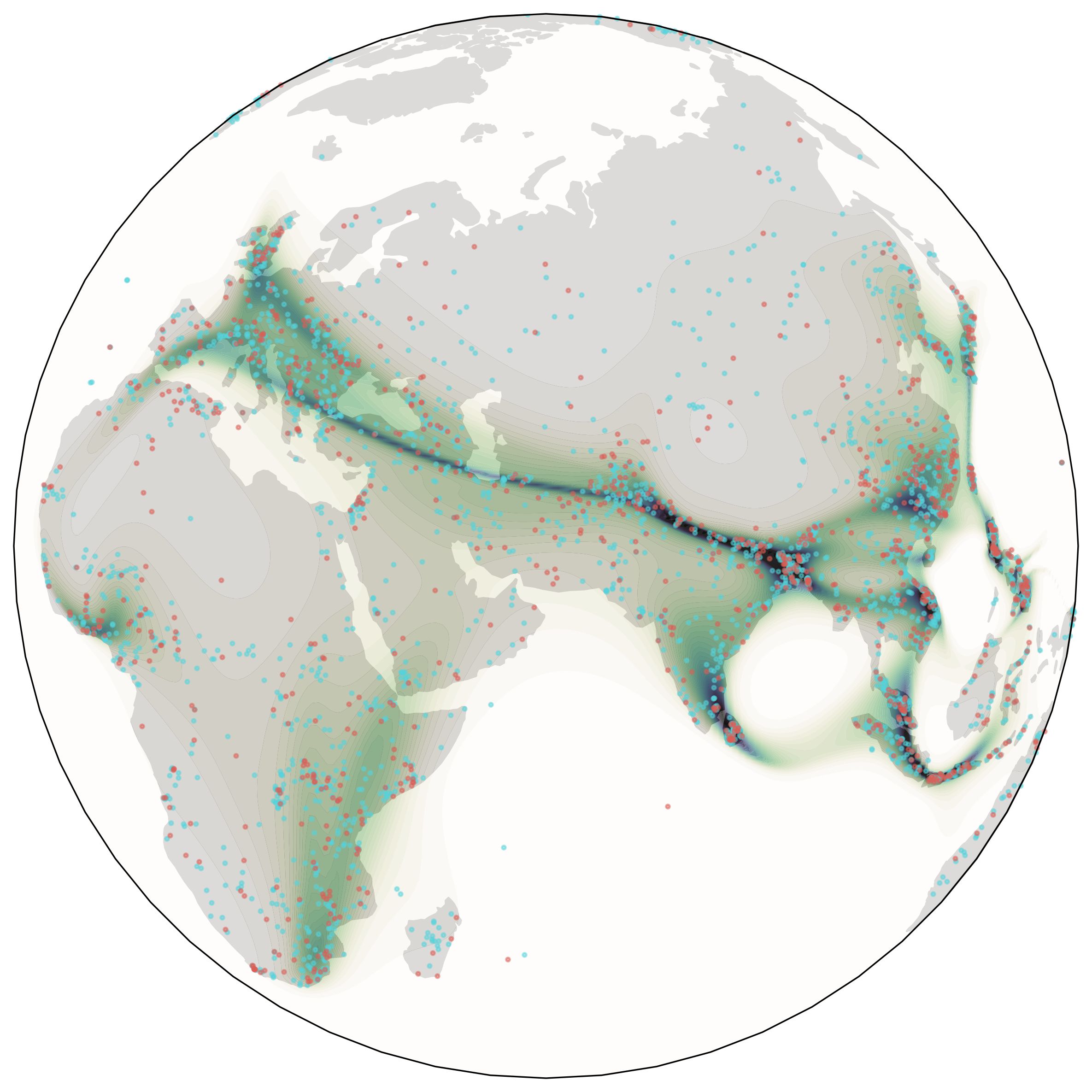}
\end{subfigure}\hfil
\begin{subfigure}{.33\textwidth}
  \includegraphics[width=\linewidth]{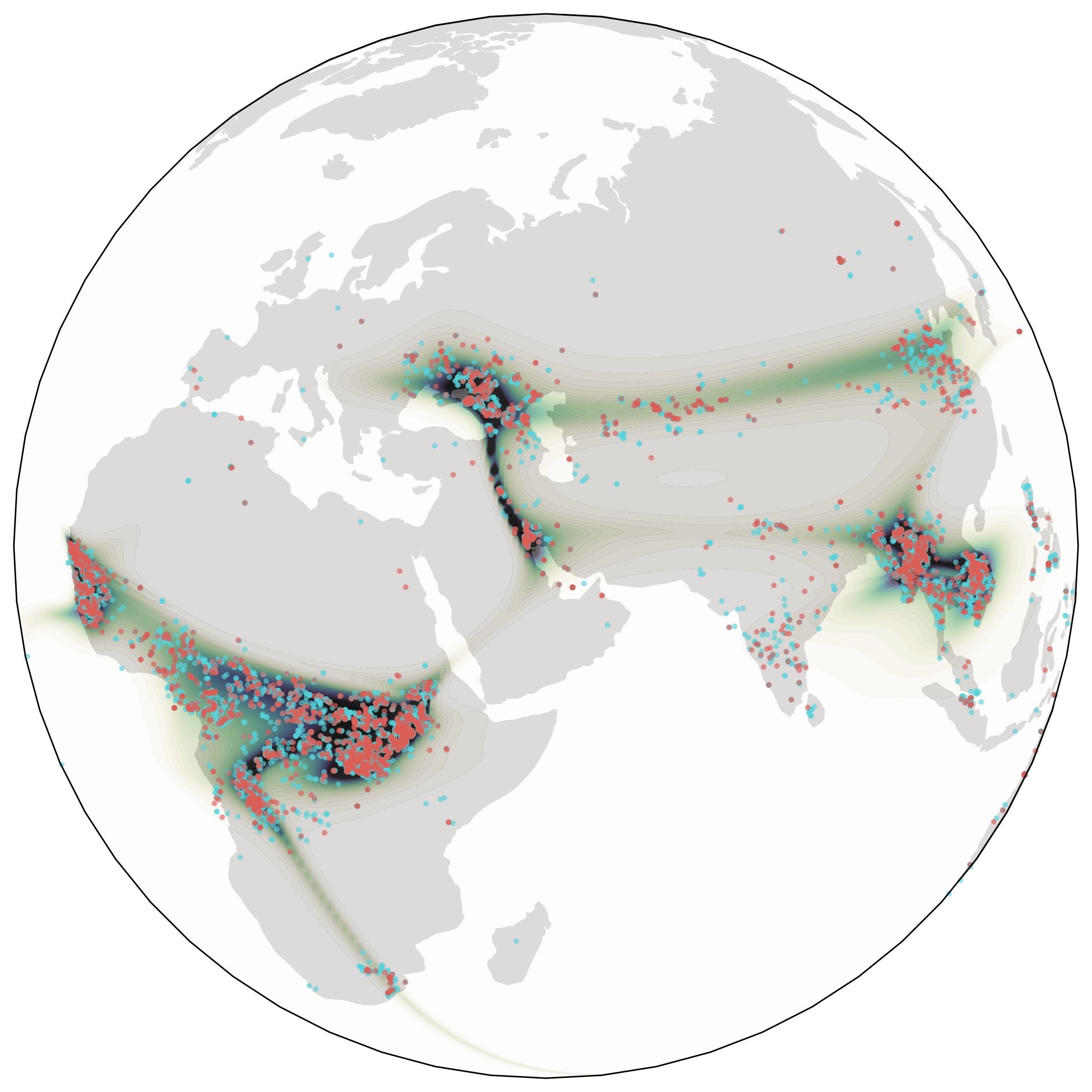}
\end{subfigure}\hfil
\begin{subfigure}{.33\textwidth}
  \includegraphics[width=\linewidth]{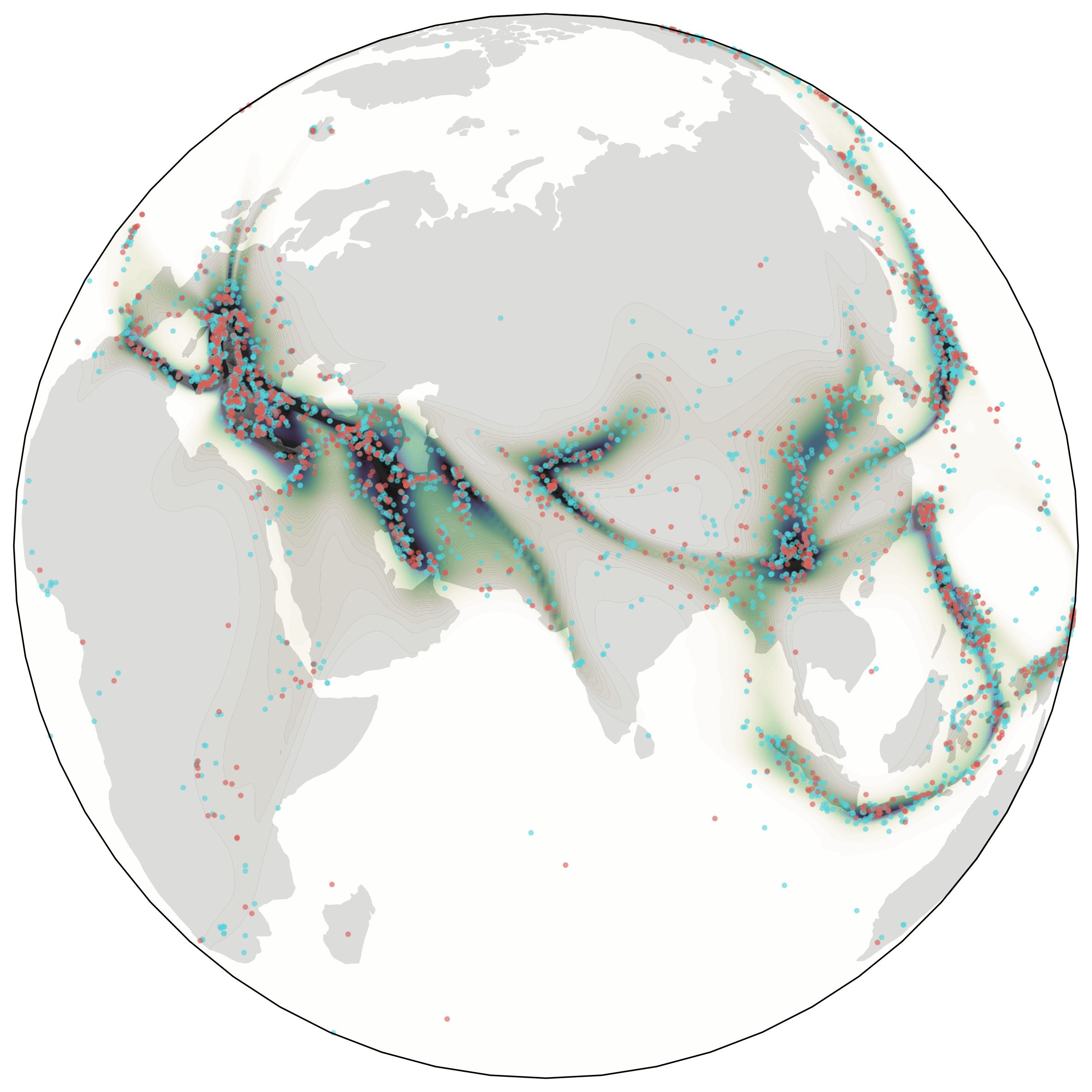}
  \put(-150,40){\rotatebox{90}{Riemannian}}
  \put(-90,-10){Earthquake}
\end{subfigure}\hfil
\begin{subfigure}{.33\textwidth}
  \includegraphics[width=\linewidth]{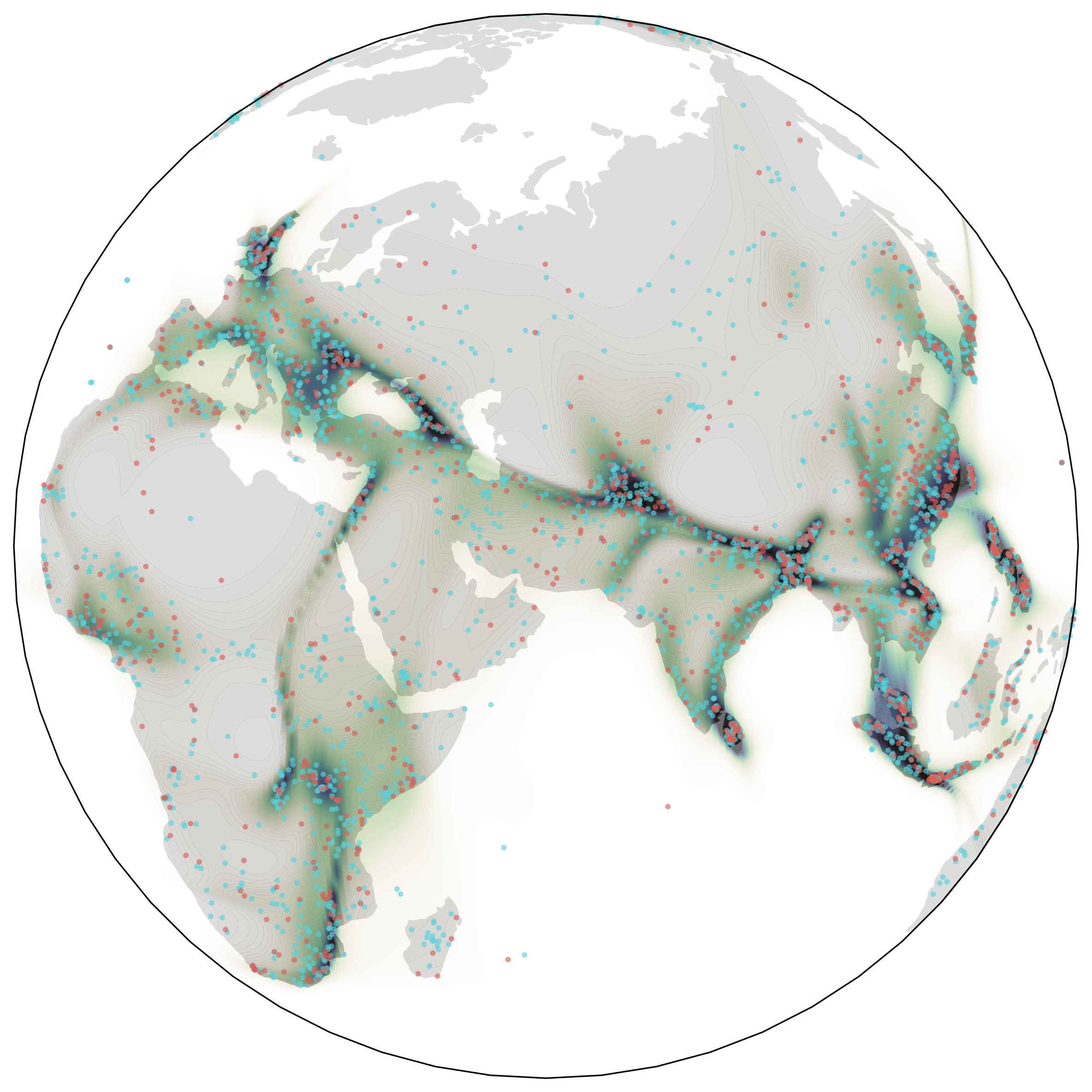}
  \put(-80,-10){Flood}
\end{subfigure}\hfil
\begin{subfigure}{.33\textwidth}
  \includegraphics[width=\linewidth]{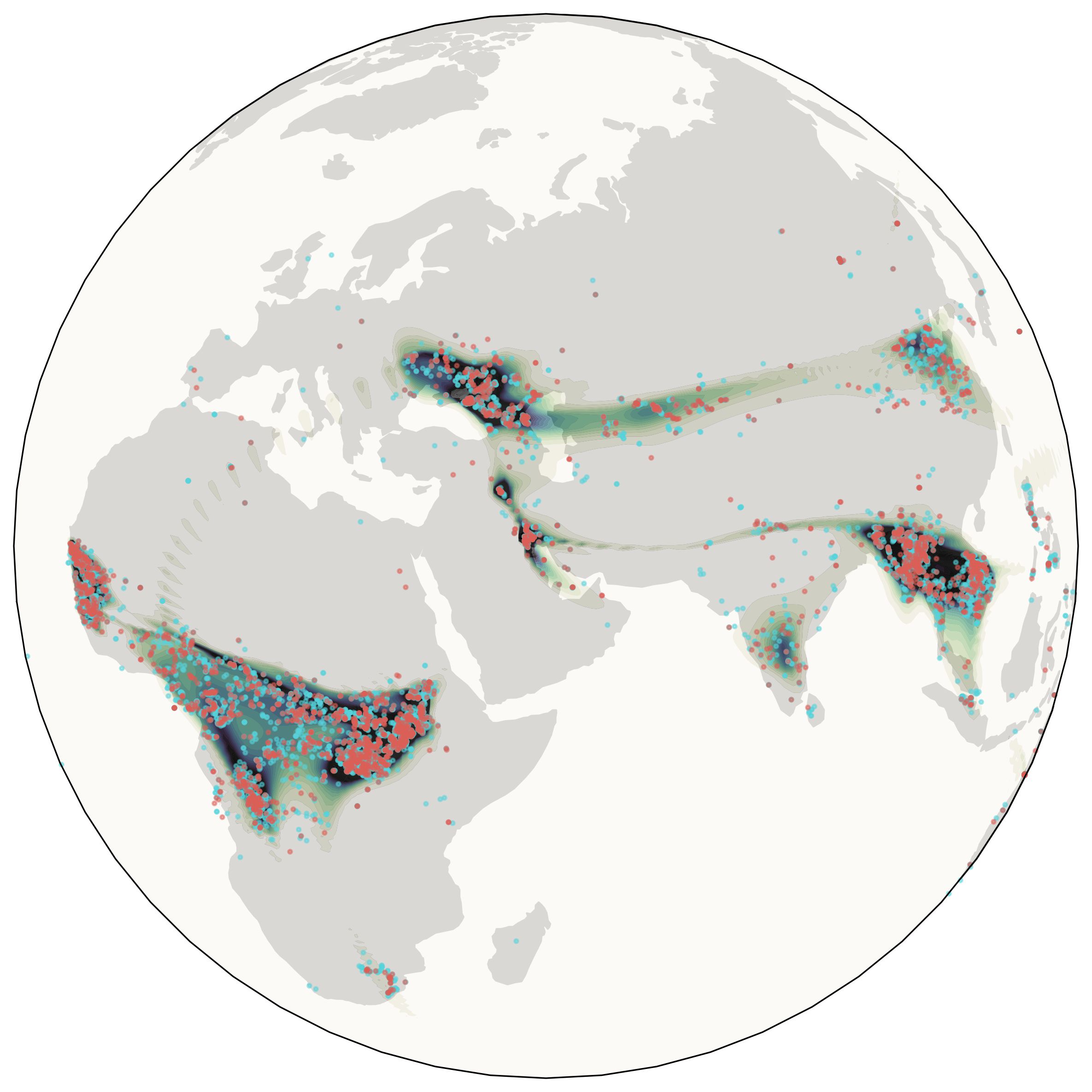}
  \put(-70,-10){Fire}
\end{subfigure}
\caption{
    Density estimation for earth sciences data.
    Blue and red dots represent training and testing datapoints, respectively.
    Heatmaps depict the log-likelihood of the trained models.
  }
  \label{fig:earth}
\end{figure}
%


\paragraph{Limitations}
In the spherical setting, the stochastic estimator to approximate the divergence
from \cref{eq:stochastic_div} exhibits high variance. Its variance scales with
the inverse of $\sqrt{|G(\z)|}=\sin(\theta)$, which becomes too large around the
north pole and thus requires the use of the exact estimator. On large-scale
datasets, where the stochastic estimator has important runtime advantages, this
issue could be alleviated by choosing a different vector field basis than the
one induced by some local coordinates \citep[e.g.][]{falorsi2020Neural}. For the
Poincar\'e ball, no such variance behavior of the stochastic estimator exists
and it can readily be applied to large-scale data.

Compared to a well-optimized linear layer, the use of the geodesic distance
layer (see \cref{sec:vector_field}) induces an extra
computational cost as shown in \cref{fig:disk_ablation_compute}. Empirically, the geodesic layer helps to improve performance
in the hyperbolic setting but had less of an effect in the spherical setting. As such, the geodesic layer can be regarded as an optional component that can improve the quality of the model at an additional computational cost.

\section{Discussion}

In this paper we proposed a principled way to parametrize expressive probability distributions on Riemannian manifolds.
Specifically, we introduced \emph{Riemmanian continuous normalizing flows} in which flows are defined via vector fields on manifolds and computed as the solution to the associated \gls{ODE}.
We empirically demonstrated that this method can yield substantial improvements when modelling data on constant curvature manifolds compared to conventional or projected flows.


\newpage

\section*{Broader impact}
The work presented in this paper focuses on the learning of well-specified
probabilistic models for manifold-valued data. Consequently, its applications are
especially promising to advance scientific understanding in fields such as earth
and climate science, computational biology, and computer vision. As a
foundational method, our work inherits the broader ethical aspects and future
societal consequences of machine learning in general.

\section*{Acknowledgments}
We are grateful to Adam Foster, Yann Dubois, Laura Ruis, Anthony Caterini, Adam Golinski, Chris Maddison, Salem Said, Alessandro Barp, Tom Rainforth and Yee Whye Teh for fruitful discussions and support.
EM research leading to these results received funding from the European Research Council under the European Union’s Seventh Framework Programme (FP7/2007- 2013) ERC grant agreement no. 617071 and he acknowledges Microsoft Research and EPSRC for funding EM’s studentship.

\bibliography{clean_bib,bib_data}

\newpage
\clearpage
\begin{appendices}
  \appendixhead
  \appendix

\section{Constant curvature manifolds}
In the following, we provide a brief overview of Riemannian geometry and constant curvature
manifolds, specifically the Poincar\'e ball and the hypersphere models. We will use
$\left\|\cdot\right\|$ and $\langle \cdot, \cdot \rangle$ to
denote the Euclidean norm and inner product. For norms
and inner products on tangent spaces $\T_{\z}\M$, we write
$\left\|\cdot\right\|_{\z}$ and $\langle \cdot, \cdot \rangle_{\z}$ where
$\z \in \M$.
\subsection{Review of Riemannian geometry} \label{sec:riem_review}
A real, smooth \emph{manifold} $\M$ is a set of points $\z$, which is "locally similar" to a linear space.
For every point $\z$ of the manifold $\M$ is attached a real vector space of the same dimensionality as $\M$ called the \emph{tangent space} $\T_{\z}\M$.
Intuitively, it contains all the possible directions in which one can tangentially pass through $\z$.
Taking the disjoint union of all tangent spaces yields the \emph{tangent bundle} $\T\M = \cap_{\z \in \M} \T_{\z}\M$.
For each point $\z$ of the manifold, the \emph{metric tensor} $\g(\z)$ defines an inner product on the associated tangent space
as $\g(\z) = \langle\cdot,\cdot \rangle_{\z}: \T_{\z}\M \times \T_{\z}\M \rightarrow \R$.
The \emph{matrix representation of the Riemannian metric} $G(\z)$, is defined such that \[\forall \bm{u},\bm{v} \in \T_{\z}\M\times\T_{\z}\M, \ \langle\bm{u},\bm{v}\rangle_{\z} = \g(\z)(\bm{u},\bm{v}) = \bm{u}^T G(\z) \bm{v}.\]
A \emph{Riemannian manifold} is then given as a tuple $(\M, \g)$ \citep{petersen2006Riemannian}.
The metric tensor gives a \emph{local} notion of angle, length of curves, surface area and volume, from which \emph{global} quantities can be derived by integrating local contributions.
A norm is induced by the inner product on $\T_{\z}\M$: $\left\| \cdot \right\|_{\z}=\sqrt{\langle\cdot,\cdot\rangle_{\z}}$. 
An infinitesimal volume element is induced on each tangent space $\T_{\z}\M$, and thus a measure $d\text{Vol}(\z) = \sqrt{|G(\z)|}~d\text{Leb}(\z)$ on the manifold, with $\text{Leb}(\z)$ being the Lebesgue measure.
The length of a curve $\gamma: t \mapsto \gamma(t) \in \M$ is given by 
$L(\gamma)=\int_{0}^{1}{\|\gamma'(t)\|_{\gamma(t)} dt}$.
The concept of straight lines can then be generalized to \emph{geodesics}, which are constant speed curves giving the shortest path between pairs of points $\z, \y$ of the manifold: $\gamma^* = \argmin L(\gamma)$ with $\gamma(0)=\z$, $\gamma(1)=\y$ and $\left\|\gamma'(t)\right\|_{\gamma(t)}=1$.
A \emph{global} distance is thus induced on $\M$ given by \[d_{\M}(\z, \y) = \inf L(\gamma).\]
Endowing $\M$ with that distance consequently defines a metric space $(\M, d_\M)$.
The concept of moving along a "straight" curve with constant velocity is given by the \emph{exponential map}.
In particular, there is a unique unit speed \emph{geodesic} $\gamma$ satisfying $\gamma(0) = \z$ with initial tangent vector $\gamma'(0) = \v$.
The corresponding exponential map is then defined by $\exp_{\z}(\bm{v}) = \gamma(1)$.
The \emph{logarithm map} is the inverse $\log_{\z} = \exp_{\z}^{-1}: \M \rightarrow \T_{\z}\M$.
The $\exp_{\z}$ map is well-defined on the full tangent space $\T_{\z}\M$ for all $z \in \M$ if and only if $\M$ is geodesically complete, i.e.\ if all geodesics can "run" indefinitely.
This is the case for the Poincar\'e ball and hypersphere. 
\subsection{The Poincar\'e ball model of hyperbolic geometry} \label{sec:poincare_ball}
In the following, we provide a brief overview of key concepts related to
hyperbolic geometry.
A $d$-dimensional \emph{hyperbolic} space is a complete, simply connected, $d$-dimensional Riemannian manifold with \emph{constant negative curvature} $K$.
The \emph{Poincar\'e ball} is one model of this geometry, and is formally defined as the Riemannian manifold $\B_K^d = (\mathcal{B}_K^d, \g_K)$.
Here $\mathcal{B}_K^d$ denotes the open ball of radius $1/\sqrt{|K|}$, and $\g_K$ the \emph{metric tensor} $\g_K(\z) = ({\lambda_{\z}^K)}^2 ~\g^e(\z)$, 
where $\lambda_{\z}^K = \frac{2}{1 + K\left\|\z\right\|^2} $ and $\g^e$ denotes the Euclidean metric tensor, i.e.\ the usual dot product.
The induced \emph{invariant measure} $\text{Vol}$ is absolutely continuous with respect to the Lebesgue measure $\text{Leb}$, and its density is given by $\frac{d\text{Vol}}{d\text{Leb}}(\z) = \sqrt{|G(\z)|} = (\lambda_{\z}^K)^d$ for all $\z \in \B_K^d$.
%
%
As motivated by \cite{skopek2019Mixedcurvature}, the Poincar\'e ball $\B^d_K$ can conveniently be described through the formalism of \emph{gyrovector spaces} \citep{ungar2008gyrovector}.
These can be seen as an analogy to the way vector spaces are used in Euclidean geometry, but in the non-Euclidean geometry setting.
In particular, the \emph{Möbius addition} $\oplus_K$ of $\z, \y$ in $\B^d_K$ is defined as
\begin{align*}
\z \oplus_K \y = \frac{(1 - 2 K \left\langle\z,\y\right\rangle - K \|\y\|^2)\z + (1 + K\|\z\|^2)\y}{1 - 2 K \left\langle\z,\y\right\rangle + K^2 \|\z\|^2 \|\y\|^2}.
\end{align*}
Then the \emph{exponential map} can be expressed via this \emph{Möbius addition} as
\begin{align*}
\exp^K_{\z}(\v) = \z \oplus_K \left(\tanh\left( \sqrt{-K} \frac{\lambda^K_{\z}\|\v\|}{2} \right)\frac{\v}{\sqrt{-K}\|\v\|} \right)
\end{align*}
%
%
where $\x = -\z \oplus_K \y $ for all $\x,\y \in \B^d_K$.
%
%
%
\subsection{The hypersphere model of elliptic geometry}
In the following, we discuss key concepts related to positively curved spaces
known as \emph{elliptic} spaces, and in particular to the \emph{hypersphere} model.
The d-sphere, or \emph{hyperpshere}, is a compact submanifold of $\R^{d+1}$ with positive constant curvature $K$ whose support is  
defined by
$\mathcal{S}_K^d = \{ \z \in \mathbb{R^{d+1}} ~|~ \langle \z,\z \rangle = 1/K \}$.
It is endowed with the pull-back metric of the ambient Euclidean space.
\paragraph{Sphere}
In the two-dimensional setting $d=2$, we rely on polar coordinates to parametrize the sphere $\S^2$.
These coordinates consist of polar $\theta \in [0, \pi]$ and azimuth $\varphi \in [0, 2\pi)$ angles.
The ambient Cartesian coordinates are then given by $\r(\theta, \varphi) = (\sin(\theta) \cos(\varphi), \sin(\theta) \sin(\varphi), \cos(\theta))$.
We have $\sqrt{|G(\theta, \varphi)|} = \sin(\theta)$.
Applying the generic divergence formula (see Equation \ref{eq:div_unit}) yields the celebrated spherical divergence formula
\begin{align*}
\diver(\bm{g}) = 
\frac{1}{\sin(\theta)} \frac{\partial }{\partial \theta} \left( \sin(\theta) ~g^{\theta}(\theta, \varphi) \right)
+ \frac{1}{\sin(\theta)} \frac{\partial }{\partial \varphi} \left(g^{\varphi}(\theta, \varphi) \right).
\end{align*}
\paragraph{Hypersphere}
For higher dimensions, we can rely on the n-spherical coordinate system in which the coordinates consist of $d-1$ angular coordinates $\varphi_1, \dots, \varphi_{d-2} \in [0, \pi]$ and $\varphi_{d-2} \in [0, 2\pi)$ \citep{blumenson1960derivation}.
Then we have
$\sqrt{|G(\bm{\varphi})|} = \sin^{d-2}(\varphi_{1}) \sin^{d-3}(\varphi_{2}) \dots \sin(\varphi_{d-2})$.

Using the ambient cartesian coordinates, the \emph{exponential map} is given by
\begin{align*}
\exp^c_{\bm{\mu}}(\v) = \cos \left(\sqrt{K}\|\v\| \right) \bm{\mu} + \sin \left(\sqrt{K}\|\v\| \right) \frac{\v}{\sqrt{K}\|\v\|}
\end{align*}
for all $\z \in \S^d_K$ and $\v \in \T_{\z}\S^d_K$.
%

%
\section{Probability measures on Riemannian manifolds}
\label{sec:measures_riem}
In what follows, we discuss core concepts of probability measures on Riemannian
manifolds and show how naive methods lead to ill- and mis-specified models on
manifolds.

Probability measures and random vectors can intrinsically be defined on Riemannian manifolds so as to model uncertainty on non-flat spaces \citep{pennec2006Intrinsic}.
The Riemannian metric $G(\z)$ induces an infinitesimal volume element on each tangent space $\T_{\z}\M$, and thus a measure on the manifold,
\begin{align*}
d\vol(\z) = \sqrt{|G(\z)|} ~d\leb,
\end{align*}
with $\leb$ the Lebesgue measure.
Manifold-valued random variables would naturally be characterized by the \emph{Radon-Nikodym derivative} of a measure $\nu$ w.r.t. the Riemannian measure $\vol$ (assuming absolute continuity)
\begin{align*}
p(\z) = \frac{d\nu}{d\vol}(\z).
\end{align*}
\subsection{Ambient Euclidean probability distributions} \label{sec:ambient_distribution}
Unfortunately, conventional probabilistic models implicitly assume a \emph{flat} geometry.
This in turn cause these models to either be misspecified or ill-suited to fit manifold distributions.
Below we discuss the reasons why.

Let $P_{\mathcal{D}}$ be a target probability measure that we aim to approximate, and which is defined on a $d$-dimensional manifold $\M \subseteq \R^D$.
Furthermore, we assume it admits a Radon-Nikodym derivative $p_{\mathcal{D}}$ with respect to the manifold invariant measure $\vol$, denoting $P_{\mathcal{D}} \ll \vol$ with $\ll$ denoting absolute continuity.
Conventional normalizing flows implicitly assume the parametrized probability measure $P_{\theta}$ to have support on the ambient space $\R^D$ and to be absolutely continuous with respect to the Lebesgue measure $\leb_{\R^D}$.
We denote its density by $p_{\theta}$.

Next, assume $D=d$, such as for $\M=\B^d \subseteq \R^d$.
With $\z$ the $d$-dimensional Cartesian coordinates, we have
$\frac{d\vol}{d\leb_{\R^d}}(\z) = \sqrt{|G(\z)|}$.
One could then see the manifold-valued target $P_{\mathcal{D}}$ as being a probability measure on $\R^d$ with a density w.r.t. the Lebesgue measure given by 
\[ \frac{dP_{\mathcal{D}}}{d\leb_{\R^d}}(\z)=p_{\mathcal{D}}(\z)\sqrt{|G(\z)|} \triangleq \tilde{p}_{\mathcal{D}}(\z).\]
In general $P_{\mathcal{D}} \ll P_{\theta}$ which implies that the \emph{forward} Kullback-Leibler divergence, or \emph{negative log-likelihood} up to constants, is defined and given by
\begin{align*}
 \mathcal{L}^{\text{Like}}(\theta)
= \KL{P_{\mathcal{D}}}{P_{\theta}} + \mathbb{H}(P_{\mathcal{D}})
= \mathbb{E}_{P_{\mathcal{D}}} \left[ \log \left(\frac{\tilde{p}_{\mathcal{D}}(\z)}{p_{\theta}(\z)} \right) \right] + \mathbb{H}(P_{\mathcal{D}})
 = - \mathbb{E}_{P_{\mathcal{D}}} \left[ \log \frac{p_{\theta}(\z)}{\sqrt{|G(\z)|}} \right].
\end{align*}
Minimising $\mathcal{L}^{\text{Like}}(\theta)$ amounts to pushing-forward $P_{\theta}$'s mass so that empirical observations $\z_i \sim P_{\mathcal{D}}$ have a positive likelihood under $P_{\theta}$.
Yet, in general the model $(P_{\theta})$ has (most of his) mass outside the manifold's support which may cause such a \emph{naive} approach to be ill-suited.
More crucially it implies that in general the model's mass is \emph{not} covering the full target's support.
In that case, the \emph{reverse} Kullback-Leibler divergence $ \KL{P_{\theta}}{P_{\mathcal{D}}}=\mathcal{L}^{\text{KL}}(\theta)$ is not even defined.

Next, consider the case where $\M$ is a submanifold embedded in $\R^D$ with $D > d$, such as $\M=\S^d$ where $D=d+1$.
In this setting the \emph{naive} model $P_{\theta}$ is even \emph{misspecified} since it is defined on a different probability space than the target.
In the limit $\supp(P_{\theta}) \rightarrow \M$, $ P_{\theta}$ is not defined because we have that $\int_{\R^D} P_{\theta} \rightarrow \infty$.
The target does consequently not belong to the model's class.
%
%
\section{Instantaneous change of variable} \label{sec:proof}
In the following we derive the \emph{instantaneous change of density} that a manifold-valued random variable induces when its dynamics are governed by an \gls{ODE}.
We show that in the Riemannian setting this \emph{instantaneous change of density} can be expressed in terms of the manifold's metric tensor.

\paragraph{Proof of \cref{prop:instant_change}}
%
\begin{proof}
For a time dependant particles $\z(t)$, whose dynamics are given by the following \gls{ODE}
\begin{align*}
\frac{d \z(t)}{d t}  = \f(\z(t), t)
\end{align*}
the change in density is given by the Liouville equation (or Fokker–Planck equation without the diffusion term); $\forall \z \in \M, \forall t \in [0, T]$
\begin{align*}
\frac{\partial}{\partial t} p(\z, t) &= - \diver(p(\z, t) \f(\z, t)) \\
 &= - \Bigl< \frac{\partial}{\partial \z} p(\z, t), \f(\z, t) \Bigr>_{\z} - p(\z, t) ~\diver( \f(\z, t))  \nonumber
\end{align*}
where the last step was obtained by applying the divergence product rule.
By introducing the time dependence in $\z(t)$ and differentiating with respect to time we get
\begin{align*}
\frac{\partial}{\partial t} p(\z(t), t) &= 
\Bigl< \frac{\partial}{\partial \z} p(\z(t), t), \frac{\partial}{\partial t} \z(t) \Bigr>_{\z(t)} 
 + \frac{\partial}{\partial t} p(\z(t), t) \nonumber \\
&= \Bigl< \frac{\partial}{\partial \z} p(\z(t), t), \f(\z(t), t) \Bigr>_{\z(t)} 
 - \Bigl< \frac{\partial}{\partial \z} p(\z(t), t), \f(\z(t), t) \Bigr>_{\z(t)} 
 - p(\z(t), t) ~\diver( \f(\z(t), t)) \nonumber \\
&= - p(\z(t), t) ~\diver( \f(\z(t), t)) \nonumber
\end{align*}
Hence the evolution of the log density is given by
\begin{align} \label{eq:app_continuous_change_of_var}
\frac{\partial}{\partial t} \log p(\z(t), t) 
&= - \diver(\f(\z(t), t)).
\end{align}
\end{proof}
\paragraph{Divergence computation} \label{sec:div_computation}
For a Riemannian manifold $(\M, g)$, with local coordinates $\z$, the divergence of a vector field $\f$ is given by
\begin{align} \label{eq:div_beg}
 \diver(\f(\z, t)) &= \frac{1}{\sqrt{|G(\z)|}} \sum_{i=1}^d \frac{\partial}{\partial z^i} 
 \left(\sqrt{|G(\z)|} ~f^i(\z, t) \right) \\
 &= \frac{1}{\sqrt{|G(\z)|}} \sum_{i=1}^d  \left( 
 \sqrt{|G(\z)|} ~\frac{\partial}{\partial z^i} f^i(\z, t) +
 f^i(\z, t) ~\frac{\partial}{\partial z^i} \sqrt{|G(\z)|} \right) \nonumber \\
 &= \sum_{i=1}^d \frac{\partial}{\partial z^i} f^i(\z, t) + 
 \frac{1}{\sqrt{|G(\z)|}} \sum_{i=1}^d f^i(\z, t)~\frac{\partial}{\partial z^i} \sqrt{|G(\z)|}  \nonumber \\
 &= \tr \left( \frac{\partial}{\partial \z}\f(\z, t) \right)  + 
 \frac{1}{\sqrt{|G(\z)|}} \Bigl< \f(\z, t), \frac{\partial}{\partial \z} \sqrt{|G(\z)|}~ \Bigr>. \label{eq:div}
 \end{align}
We note that in Equation \ref{eq:div_beg}, $f_i$ are the components of the vector field $\f$ with respect to the local unnormalized covariant basis $\left(\e_i\right)_{i=1}^d = \left(\left( \frac{\partial}{\partial z^i} \right)_{\z} \right)_{i=1}^d$.
However it is convenient to work with local basis having unit length vectors.
If we write ${\hat{\e}}_{i}$ for this normalized basis, and 
$\hat{f}^i$ for the components of $\f$ with respect to this normalized basis, we have that
\begin{align*}
\f = \sum_i f^i ~\e_i
 = \sum_i f^i ~ \|\e_i\| ~\frac{\e_i}{\|\e_i\|}
 = \sum_i f^i \sqrt{G_{ii}} ~\frac{\e_i}{\|\e_i\|}
 = \sum_i \hat{f}^i ~\hat{\e}_i
\end{align*}
using one of the properties of the metric tensor.
By dotting both sides of the last equality with the contravariant element $\hat{\e}_i$ we get that $\hat{f}^i = f^i \sqrt{G_{ii}}$.
Substituting in Equation \ref{eq:div_beg} yields
\begin{align} \label{eq:div_unit}
 \diver \left(\hat{\f}(\z, t)\right) &= \frac{1}{\sqrt{|G(\z)|}} \sum_{i=1}^d \frac{\partial}{\partial z^i}
  \left( \sqrt{\frac{|G(\z)|}{G_{ii}(\z)}} ~\hat{f}^i(\z, t) \right).
 \end{align}
Combining \cref{eq:app_continuous_change_of_var,eq:div_unit} and we finally get
 \begin{align} \label{eq:app_continuous_change_of_var_2}
\frac{\partial \log p(\z(t), t)}{\partial t}  = - \frac{1}{\sqrt{|G(\z)|}} \sum_{i=1}^d \frac{\partial}{\partial z^i} \left( \sqrt{\frac{|G(\z)|}{G_{ii}(\z)}} ~\hat{f}^i(\z, t) \right).
 \end{align}
We rely on this \cref{eq:app_continuous_change_of_var_2} for practical numerical experiments.
\section{regularization} \label{sec:regularization}
\subsection{$l^2$-norm}  \label{sec:regularization_l2}
Henceforth we motivate the use of an $l^2$ norm regularization in the context of \acrlongpl{CNF}.
We do so by highlighting a connection with the dynamical formulation of \acrlong{OT}, and by proving that this formulation still holds in the manifold setting.
\paragraph{Monge-Kantorovich mass transfer problem}
Let $(\M, d_{\M})$ be a metric space, and $c: \M \times \M \rightarrow [0, +\infty)$ a measurable map.
Given probability measures $p_0$ and $p_T$ on $\M$, Monge's formulation of the optimal transportation problem is to find a transport map $\phi^*: \M \rightarrow \M$ that realizes the infimum
\begin{align*}
\inf_{\phi} \int_{\M} c(\phi(\z), \z) ~p_0(d\z) \ \text{s.t.} \ p_T \sharp = p_0.
\end{align*}
It can be shown that this yields a metric on probability measures, and for $c=d_{\M}^2$, it is called the $L^2$ Kantorovich (or Wasserstein) distance
\begin{align} \label{eq:wasserstein_distance}
d_{W^2}(p_0, p_T)^2 = \inf_{\phi} \int_{\M} d_{\M}(\phi(\z), \z)^2 ~p_0(d\z).
\end{align}
By reintroducing the time variable in the ${L}^2$ Monge-Kantorovich mass transfer problem, the optimal transport map $\phi^*$ can be reformulated as the generated flow from an optimal vector field $\f$.
\begin{proposition}[Dynamical formulation from \citep{benamou2000computational}] \label{prop:dynamical_ot}
Indeed we have
\begin{align}  \label{eq:path-length}
d_{W^2}(p_0, p_T)^2 = 
\inf \frac{1}{T} \int_{0}^T \| \f \|^2_{p_t} dt 
 = \inf \frac{1}{T} \int_{0}^{T} \int_{\M} \langle \f(\z, t), \f(\z, t) \rangle_{\z} ~p_t(d\z) ~dt 
\end{align}
where the infimum is taken among all weakly continuous distributional solutions of the continuity equation $\frac{\partial}{\partial t} p_t = - \diver(p_t \f)$ such that $p(0) = p_0$ and $p(T) = p_T$.
Writing $\phi^*_t = \phi^*(\cdot, t)$ the flow generated by the optimal \gls{ODE}, then the optimal transport map is given by $\phi^* = \phi^*_T$.
\end{proposition}

The RHS of Equation (\ref{eq:path-length}) can then be approximated with no extra-cost with a \acrlong{MC} estimator given samples from $p_t = \phi_t \sharp p_0$.

\paragraph{Manifold Setting}
Let's now focus on the setting where $\M$ is a Riemannian manifold.

\begin{proposition}[Optimal map \citep{ambrosio2003Optimal}] 
Assume that $\M$ is a $C^3$, complete Riemannian manifold with no boundary and $d_{\M}$ is the Riemannian distance. If $p_0$, $p_T$ have finite second order moments and $p_0$ is absolutely continuous with respect to $vol_{\M}$, then there exists a unique optimal transport map $\phi$ for the Monge-Kantorovich problem with cost $c = d_{\M}^2$.
Moreover there exists a potential $h : \M \mapsto \R$ such that
\begin{align*}
\phi^*(\z) = \exp_{\z}(- \nabla h(\z)) \quad vol_{\M}-a.e..
\end{align*}
\end{proposition}

Proposition \ref{prop:dynamical_ot} has been stated and proved for the case $\M=\R^d$.
Below we extend the proof given by \cite{benamou2000computational} for the manifold setting.
\begin{proof}[Proof of Proposition \ref{prop:dynamical_ot}]
We follow the same reasoning as the one developed for the Euclidean setting.
Let's first upper bound the Wasserstein distance, and then state the optimal flow which yields equality.
We have
\begin{align*}
\frac{1}{T} \int_{0}^{T} \int_{\M} \| \f(\z, t) \|_{\z}^2 ~p_t(d\z) ~dt
&= \frac{1}{T} \int_{0}^{T} \int_{\M} \left\| \f(\phi(\z, t)), t) \right\|_{\z}^2 ~p_0(d\z) ~dt \\
&= \frac{1}{T} \int_{0}^{T} \int_{\M} \left\| \frac{\partial}{\partial t} \phi(\z, t) \right\|_{\z}^2 ~p_0(d\z) ~dt \\
& \ge \int_{\M} d_{\M}(\phi(\z, T), \phi(\z, 0))^2 ~p_0(d\z) ~dt \\
& = \int_{\M} d_{\M}(\phi(\z, T), \z)^2 ~p_0(d\z) ~dt \\
& \ge \int_{\M} d_{\M}(\phi(\z), \z)^2 ~p_0(d\z) ~dt \\
& = d_{W^2}(p_0, p_T)^2.
\end{align*}
Thus, the optimal choice of flow $\phi$ is given by
\begin{align} \label{eq:optimal_flow}
\phi(\z, t) = \exp_{\z}\left(\frac{t}{T} \log_{\z}(\phi^*(\z))\right),
\end{align}
since $\phi(\z, 0) = \z$, $\phi(\z, T) = \phi^*(\z)$ and 
\begin{align*}
\left\| \frac{\partial}{\partial t} \phi(\z, t) \right\|_{\z} 
= \left\| \frac{\partial}{\partial t} \phi(\z, t=0) \right\|_{\z} 
= \left\| \log_{\z}(\phi^*(\z) \right\|_{\z} 
= d_{\M}(\phi^*(\z), \z).
\end{align*}
\end{proof}
%
%
Note that the optimal flow from Equation \ref{eq:optimal_flow} yields integral paths $\gamma(t) = \phi(\z, t)$ that are \emph{geodesics} and have constant \emph{velocity}.
\paragraph{Motivation}
Regularizing the vector field with the RHS of Equation \ref{eq:path-length} would hence tend to make the generated flow $\phi_T$ closer to the optimal map $\psi^*$.
By doing so, one hopes to increase smoothness of $\f$ and consequently lower the solver \gls{NFE} given a fixed tolerance.

This has been observed in the Euclidean setting by \cite{finlay2020How}.
They empirically showed that regularizing the loss of a \gls{CNF} with the vector field's $l^2$ norm improves training speed.
Motivated by the successful use of gradient regularization \citep{novak2018Sensitivity,drucker1992Improving}, they showed that additionally regularizing the Frobenius norm of the vector field's Jacobian helps.
In the following subsection we remind that this regularization term can also be motivated from an estimator's variance perspective.

\subsection{Frobenius norm}  \label{sec:regularization_frobenius}
\paragraph{Hutchinson's estimator}
Hutchinson's estimator \citep{hutchinson1990stochastic} is a simple way to obtain a stochastic estimate of the trace of a matrix.
Given a d-dimensional random vector $\bm{\epsilon} \sim p$ such that $\mathbb{E}[\bm{\epsilon}] = 0$ and $\text{Cov}(\bm{\epsilon}) = I_d$, we have
\begin{align*}
\tr(A) = \mathbb{E}_{\bm{\epsilon} \sim p}[\bm{\epsilon}^T A \bm{\epsilon}].
\end{align*}
Rademacher and Gaussian distributions have been used in practice.
For a Rademacher, the variance is given by \citep{avron2011Randomized}
\begin{align*}
\mathbb{V}_{\bm{\epsilon} \sim p}[\bm{\epsilon}^T A \bm{\epsilon}] =
2 ~\|A\|_F - 2 \sum_i A^2_{ii},
\end{align*}
whereas for a Gaussian it is given by
\begin{align*}
\mathbb{V}_{\bm{\epsilon} \sim p}[\bm{\epsilon}^T A \bm{\epsilon}] =
2 ~\|A\|_F.
\end{align*}

\paragraph{Divergence computation}
As reminded in Appendix \ref{sec:div_computation} by Equation \ref{eq:div}, computing the vector field divergence $\diver(\f(\z, t))$ involves the computation of the trace of vector field's Jacobian
$\tr \left( \frac{\partial}{\partial \z}\f(\z, t) \right)$.
As highlighted in \cite{grathwohl2018FFJORD,salman2018Deep}, one can rely on the Hutchinson's estimator to estimate this trace with $A = \frac{\partial}{\partial \z}\f(\z, t)$.

The variance of this estimator thus depends on the Frobenius norm of the vector's field Jacobian $\|\frac{\partial}{\partial \z}\f(\z, t)\|_F$, as noted in \cite{grathwohl2018FFJORD}.
Regularizing this Jacobian should then improve training by reducing the variance of the divergence estimator.

\section{Vector flows and neural architecture} \label{sec:app_vector_field}
Hereafter we discuss about flows generated by vector fields, and neural architectural choices that we make for their parametrization.
Properties of vector fields have direct consequences on the properties of the generated flow and in turn on the associated pushforward probability distributions. 
In particular we derive sufficient conditions on the flow so that it is \emph{global}, i.e.\ is a bijection mapping the manifold to itself.

\subsection{Existence and uniqueness of a global flow} \label{sec:global_flow}
We start by discussing about vector flows and sufficient conditions on their uniqueness and existence.
\paragraph{Local flow}
First we remind the Fundamental theorem of flows \citep{lee2003Introduction} which gives the existence and uniqueness of a smooth \emph{local flow}.
\begin{proposition}[Fundamental theorem of flows] \label{prop:local_flow}
Let $\M$ be a \emph{smooth} complete manifold with local coordinates $\z$. 
Let $f_{\theta}: \M \times \R \mapsto \T\M~$ a $C^1$ time- dependent \emph{vector field} and $\z_0 \in \M$.
Then there exists an open interval $I$ with $0 \in I$, an open subset $U \subseteq \M$ containing $\z_0$, and a unique smooth map $\phi: I \times U \mapsto \M$ called \emph{local flow} which satisfies the following properties.
We write $\phi_t(\z) = \phi(\z, t)$.
\begin{enumerate}
	\item $\frac{\partial}{\partial t}\phi(\z, t) = f_{\theta}(\phi(\z, t), t)~$  for all $\z, t \in U \times I$, and $\phi_0 = \text{id}_{\M}$.
	\item For each $t \in I$, the map $\phi_t: U \mapsto \M$ is a local $C^1$-diffeomorphism.
\end{enumerate}
\end{proposition}
Note that with such assumptions, the existence and uniqueness of flows $\phi_t$ are only \emph{local}.

\paragraph{Global flow}

We would like the flow $\phi$ to be defined for all times and on the whole manifold, i.e.\ a \emph{global flow} $\phi: \M \times \R \mapsto \M$.
Fortunately, if $\M$ is \emph{compact} (such as n-spheres and torii), then the flow is global \citep{lee2003Introduction}.
We show below that another sufficient condition for the flow to be global is that  the vector field be bounded.

\begin{proposition}[Global Flow] \label{prop:global_flow}
Let $\M$ be a \emph{smooth} complete manifold.
Let $f_{\theta}: \M \times \R \mapsto \T\M~$ be a $C^1$ bounded time-dependent \emph{vector field}.
Then the domain of the flow $\phi$ is $\R \times \M$, i.e.\ the flow is \emph{global}.
\end{proposition}

\begin{corollary}
For each $t \in \R$, the map $\phi_t: \M \mapsto \M$ is a $C^1$-diffeomorphism.
\end{corollary}

\begin{proof}[Proof of Proposition \ref{prop:global_flow}]
Let $c > 0$ s.t.\ $\|\f\| < c$	, and $\z_0 \in \M$ be an initial point.
Proposition \ref{prop:local_flow} gives the existence of an open interval $I=(a, b)$, a neighbourhood $U$ of $\z_0$ and a local flow $\phi: (a, b) \times U \mapsto \M$.
We write $\gamma = \phi(\z_0, \cdot)$.
The maximal interval of $\gamma$ is $(a, b)$, which means that $\gamma$ cannot be extended outside $(a, b)$.
Suppose that $b < \infty$.

The integral path $\gamma$ is Lipschitz continuous on $(a, b)$ since we have
\begin{align} \label{eq:path_integral_lispchitz}
d_{\M}(\gamma(t), \gamma(s)) \le \int_s^t \|\gamma'(t)\| ~dt = 
\int_s^t \|\f(\gamma(t), t)\| ~dt
\le c ~|t - s|
\end{align}
for all $s<t \in (a, b)$.

Let $(t_n)$ be a sequence in $(a, b)$ that converges to $b$.
Then since $(t_n)$ is a convergent sequence, it must also be a Cauchy sequence.
Then $\gamma(t_n)$ is also a Cauchy sequence by Equation \ref{eq:path_integral_lispchitz}.
Since $\M$ is geodesically complete, it follows by Hopf-Rinow theorem that $(\M, d_{\M})$ is complete, hence that $\gamma(t_n)$ converges to a point $\p \in \M$.

Now suppose that $(s_n)$ is another sequence in $(a, b)$ that converges to $b$.
Then by Equation \ref{eq:path_integral_lispchitz} $\lim_{n \to \infty} d(\gamma(s_n), \gamma(t_n)) = 0$, thus $\gamma(s_n)$ also converges to $\lim_{n \to \infty} \gamma(t_n) = \p$.
So for every sequence $(t_n)$ in $(a, b)$ that converges to $b$, we have that $(\gamma(t_n))$ converges to $\p$.
Therefore by the sequential criterion for limits, we have that $\gamma$ has the limit $\p$ at the point $b$.
Therefore, define $\gamma(b)=\p$ and so $\gamma$ is continuous at $b$ which is a contradiction.
\end{proof}
\subsection{Geodesic distance layer} \label{sec:geodesic_layer}
The expressiveness of \glspl{CNF} directly depends on the expressiveness of the vector field and consequently on its architecture.
Below we detail and motivate the use of a \emph{geodesic distance} layer, as an input layer for the vector field neural architecture.
\paragraph{Linear layer}
A linear layer with one \emph{neuron} can be written in the form $h_{\a, \p}(\z) = \left\langle\a, \z-\p \right\rangle$, with orientation and offset parameters $\a,\p \in \R^d$.
Stacking $l$ such neurons $h$ yields a linear layer with width $l$.
This neuron can be rewritten in the form
\begin{align*}
&h_{\a, \p}(\z) = \text{sign}\left(\left\langle\a, \z-\p \right\rangle\right) \left\|\a\right\| d_E\left(\z,H^K_{\a,\p}\right)
\end{align*}
where $H_{\a,\p} = \{\z \in \R^p~|~ \left\langle\a,\z-\p \right\rangle = 0 \} = \p + \{ \a \}^{\perp}$ is the decision hyperplane. 
The third term is the distance between $\z$ and the decision hyperplane $H^K_{\a,\p}$ and
the first term refers to the side of $H^K_{\a,\p}$ where $\z$ lies.
\paragraph{Poincar\'e ball}
\cite{ganea2018Hyperbolic} analogously introduced a neuron $f^K_{\a, \p}: \B_K^d \rightarrow \R^p$ on the Poincar\'e ball,
\begin{align} \label{eq:geodesic_layer}
h^K_{\a, \p}(\z) = \text{sign} \left( \left\langle\a, \log^K_{\p}(\z) \right\rangle_{\p} \right) \left\|\a\right\|_{\p} d^K(\z,H^K_{\a,\p})
\end{align} 
with $H^K_{\a,\p} = \left\{\z \in \B_K^d~|~ \left\langle\a,\log^K_{\p}(\z) \right\rangle = 0 \right\} =\exp^K_{\p}(\{ \a \}^{\perp})$. 
A closed-formed expression for the distance $d^K(\z,H^K_{\a,\p})$ was also derived,
$d^K(\z,H^K_{\a,\p}) = \frac{1}{\sqrt{|K|}} \sinh^{-1} \left( \frac{2 \sqrt{|K|} |\left\langle -\p \oplus_K \z, a \right\rangle |}{(1 + K \| -\p \oplus_K \z\|^2)\|\a\|} \right)$ 
in the Poincar\'e ball.
To avoid an over-parametrization of the hyperplane, we set $\p = \exp_0(\a_0)$, and
$\a = \Gamma_{\bm{0} \rightarrow \p}(\a_0)$ with $\Gamma$ parallel transport (under Levi-Civita connection).
We observed that the term $\left\|\a\right\|_{\p}$ from Equation \ref{eq:geodesic_layer} was sometimes causing numerical instabilities, and that when it was not it also did not improve performance.
We consequently removed this scaling term.
The hyperplane decision boundary $H^K_{\a,\p}$ is called \emph{gyroplane} and is a semi-hypersphere orthogonal to the Poincar\'e ball's boundary.

\paragraph{Hypersphere}
In hyperspherical geometry, geodesics are great circles which can be parametrized by a vector $\w \in \R^{d+1}$ as
$H_{\w} = \{\z \in \S^d ~|~  \langle \w, \z \rangle =0 \}$.
The geodesic distance between $\z \in \S^d$ and the hyperplane $H_{\w}$ is then given by
\begin{align*}
d(\z, H_{\w}) = \left| \sin^{-1} \left( \frac{\langle \w, \z \rangle}{\sqrt{\langle \w, \w \rangle}} \right) \right|.
\end{align*}
In a similar fashion, a neuron is now defined by
\begin{align*}
h_{\w}(\z) = \left\|\w \right\|_2 \sin^{-1} \left( \frac{\langle \w, \z \rangle}{\sqrt{\langle \w, \w \rangle}} \right).
\end{align*} 
\paragraph{Geodesic distance layer}
One can then horizontally-stack $l$ neurons to make a \emph{geodesic distance} layer $g: \M \mapsto \R^l$ \citep{mathieu2019Continuous}.
Any standard feed-forward neural network can then be vertically-stacked on top of this layer.
\section{Extensions}  \label{sec:extensions}
\subsection{Product of manifolds} \label{sec:prod_manifolds}
Having described \glspl{CNF} for complete smooth manifolds in Section \ref{sec:model}, we extend these for product manifolds $\M = \M_1 \times \dots \times \M_k$.
For instance a $d$-dimensional torus is defined as $\mathbb{T}^d = \underbrace{\S^1 \times \cdots \times \S^1}_{d}$.
Any density $p_{\theta}(\z_1, \dots, \z_K)$ can decomposed via the chain rule of probability as
\begin{align*}
p_{\theta}(\z_1, \dots, \z_K) = \prod\nolimits_{k} p_{\theta_k}(\z_k~|~\z_1, \dots, \z_{k-1})
\end{align*}
where each conditional $p_{\theta_k}(\z_k~|~\z_1, \dots, \z_{k-1})$ is a density on $\M_k$.
As suggested in \cite{rezende2020Normalizing}, each conditional density can be implemented via a flow $\phi_k: \M_k \mapsto \M_k$ generated by a vector field $\f_k$, whose parameters $\theta_k$ are a function of $(\z_1, \dots, \z_{k-1})$. 
Such a flow $\phi = \phi_1 \circ \cdots \circ \phi_k$ is called \emph{autoregressive} \citep{papamakarios2018Masked} and conveniently has a lower triangular Jacobian, which determinant can be computed efficiently as the product of the diagonal term.



\section{Experimental details} \label{sec:exp_details}

Below we fully describe the experimental settings used to generate results introduced in \cref{sec:experiments}.
\paragraph{Architecture}
The architecture of the \emph{vector field} $\f_\theta$ is given by a \gls{MLP} with $3$ hidden layers and $64$ hidden units -- as in \citep{grathwohl2018FFJORD} -- for \emph{projected} (e.g.\ stereographic and wrapped cf \cref{sec:projected_methods}) and \emph{naive} (cf \cref{sec:ambient_distribution}) models.
We rely on tanh activation.
For our \gls{RCNF}, the input layer of the \gls{MLP} is replaced by a \emph{geodesic distance} layer \citep{ganea2018Hyperbolic,mathieu2019Continuous} (see \cref{sec:geodesic_layer}).

\paragraph{Objectives}
We consider two objectives, a \gls{MC} estimator of the negative log-likelihood
\begin{align*}
  \hat{\mathcal{L}}^{\text{Like}}(\theta) = - \sum_{i=1}^B \log p_\theta(\z_i) \ \text{with} \ \z_i \sim P_\mathcal{D}
\end{align*}
and a \gls{MC} estimator of the reverse \gls{KL} divergence
\begin{align*}
  \hat{\mathcal{L}}^{\text{KL}}(\theta) =  \sum_{i=1}^B \log p_\theta(h_\theta(\epsilon_i)) - \log p_\mathcal{D}(h_\theta(\epsilon_i))
\end{align*}
 with $\z_i \sim P_\theta$ being reparametrized as $\z_i=h_\theta(\epsilon_i)$ and $\epsilon_i \sim P$.

\paragraph{Optimization}
All models are trained by the stochastic optimizer Adam~\citep{kingma2015Adam} with parameters $\beta_1 = 0.9$, $\beta_2 = 0.999$, batch-size of $400$ data-points and a learning rate set to $1e^{-3}$.

\paragraph{Training}
We rely on the Dormand-Prince solver \citep{dormand1980family}, an adaptive Runge-Kutta 4(5) solver, with absolute and relative tolerance of $1e-5$ to compute approximate numerical solutions of the \gls{ODE}.
Each solver step is projected onto the manifold.
Models are trained on a cluster of GeForce RTX 2080 Ti GPU cards.

\subsection{Hyperbolic geometry and limits of conventional and wrapped methods}

In this experiment the target is set to be a wrapped normal on $\B^2$ \citep{nagano2019Wrapped,mathieu2019Continuous} with density 
$\mathcal{N}^{\text{W}}(\exp_{\bm{0}}(\alpha~\partial x), \Sigma)
= \exp_{\bm{\mu} \sharp} \mathcal{N}(\alpha~\partial x, \Sigma)$ with $\Sigma = \diag(0.3, 1.0)$.
The scalar parameter $\alpha$ allows us to locate the target closer or further away from the origin of the disk.
Through this experiment we consider three \glspl{CNF}:
\begin{itemize}
  \item \emph{Naive}: $P^{\text{N}}_{\theta} = \phi^{\R^2}_\sharp \mathcal{N}(0, 1)$
  \item \emph{Wrapped}: $P^{\text{W}}_{\theta} = (\exp_{\bm{0}} \circ ~\phi^{\R^2})_\sharp \mathcal{N}(0, 1)$
  \item \emph{Riemannian}: $P^{\text{R}}_{\theta} = \phi^{\B^2}_\sharp \mathcal{N}^{\text{W}}(0, 1)$
\end{itemize}
with $\phi^{\R^2}$ a conventional \gls{CNF} on $\R^2$, $\phi^{\B^2}$ our \gls{RCNF} introduced in \cref{sec:model}, $\mathcal{N}(0, 1)$ the standard Gaussian and $\mathcal{N}^{\text{W}}(0, 1)$ the standard wrapped normal.
For the \gls{RCNF} we scale the vector field as 
\begin{align*}
  \f_{\theta}(\z) = {|G(\z)|}^{-1/2} ~\texttt{neural\_net}(\z) = \left(\frac{1 - \|\z\|^2}{2}\right)^2 \texttt{neural\_net}(\z).
\end{align*}
These three models are trained for $1500$ iterations, by minimizing the negative log-likelihood (see \cref{fig:poicare_convergence}). The reported results are averaged over 12 runs.  
When training, the divergence is approximated by the (Hutchinson) stochastic estimator from \cref{eq:stochastic_div}.

\subsection{Spherical geometry}
Through the following spherical experiments we consider the two following models
\begin{itemize}
  \item \emph{Stereographic}: $P^{\text{S}}_{\theta} = (\rho^{-1} \circ ~\phi^{\R^2})_\sharp \mathcal{N}(0, 1)$
  \item \emph{Riemannian}: $P^{\text{R}}_{\theta} = \phi^{\S^2}_\sharp \mathcal{U}(\S^2)$
\end{itemize}
with $\rho^{-1}$ the inverse of the stereographic projection, $\phi^{\R^2}$ a conventional \gls{CNF} on $\R^2$, $\phi^{\S^2}$ our \gls{RCNF}, $\mathcal{N}(0, 1)$ the standard Gaussian and $\mathcal{U}(\S^2)$ the uniform distribution on $\S^2$.
For the \gls{RCNF} we project the output layer of the vector field as 
\begin{align*}
  \f_{\theta}(\z) = \text{proj}_{\T_{\z}\S^2} ~\texttt{neural\_net}(\z) = \frac{\texttt{neural\_net}(\z)}{\|\texttt{neural\_net}(\z)\|^2}
\end{align*}
so as to enforce output vectors to be \emph{tangent}.
All spherical experiments were performed using the exact divergence estimator.
\paragraph{Limits of the stereographic projection model}
In this experiment the target is chosen to be a $\vMF(\bm{\mu}, \kappa)$ located at $\mu= - \bm{\mu}_0$ with $- \{\bm{\mu}_0\} = (-1, 0, \dots, 0)$.
Both models are trained for $3000$ iterations by minimizing the negative log-likelihood and the reverse \gls{KL} divergence.
The reported results are averaged over 4 runs.

\paragraph{Density estimation of spherical data}
Finally we consider four earth location datasets, representing respectively volcano eruptions \citep{data_volcano}, earthquakes \citep{data_earthquake}, floods \citep{data_flood} and wild fires \citep{data_fire}.
The reported results are averaged over 12 runs.

Concerning the \glspl{CNF}, these models are trained by minimizing the negative log-likelihood for $1000$ epochs, except for the volcano eruption dataset where $3000$ epochs are required for convergence.
We observed that training models with a solver's tolerance of $1e-5$ was computationally intensive so we lowered the training tolerance to $1e-3$ for the volcano eruptions, earthquakes and wild fires datasets, and $1e-4$ for the floods dataset, while keeping it to $1e-5$ during evaluation.
We additionally did not use the geodesic layer (presented in \cref{sec:geodesic_layer}) to lower the computational time.
We also observed that annealing the learning rate such that $\alpha(t) = {0.98}^{(t / 300)} \alpha_0$ with $\alpha_0=1e^{-3}$ helped stabilizing training convergence.

Concerning the mixture of von Mises-Fisher distributions, the parameters are learned by minimizing the negative log-likelihood with Riemannian Adam \citep{becigneul2019Riemannian}.
The number of epochs is set to $10000$ for all datasets but for the volcano eruption one where $30000$ epochs are required for the vMF model to converge.
The learning rate and number of mixture components are selected by performing a hyperparameter grid search, over the following range: learning rate $\in \{1e-1, 5e-2, 1e-2\}$ and number of components $\in \{50, 100, 150, 200\}$.
\section{Additional figures} \label{sec:add_figures}
%
%
\begin{figure}[t]
  \centering
  \begin{subfigure}{1.\textwidth}
    \includegraphics[width=\linewidth]{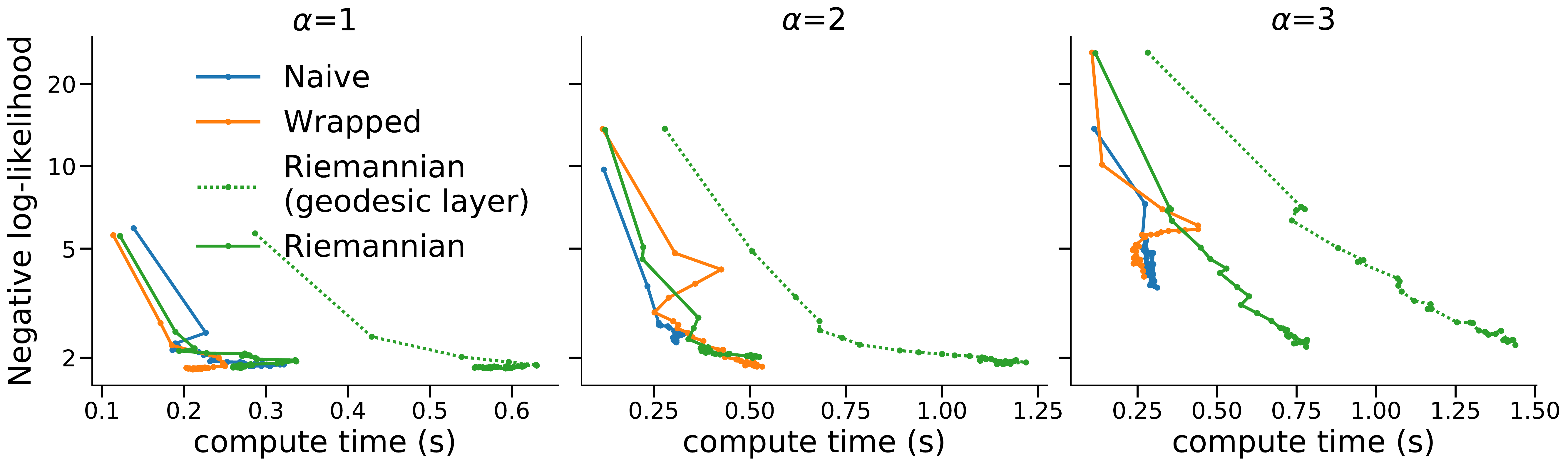}
  \end{subfigure}
    \caption{
    Ablation study of the geodesic layer computational impact for the \emph{Riemannian} model.
     Negative Log-likelihood of \emph{Riemannian} \glspl{CNF} trained to fit a $\mathcal{N}^{\text{W}}(\exp_{\bm{0}}(\alpha ~\partial x), \Sigma)$ target on $\B^2$.
    }
    \label{fig:disk_ablation_compute}
  \end{figure}

\begin{figure}[t]
\centering
\begin{subfigure}{.49\textwidth}
  \includegraphics[width=\linewidth]{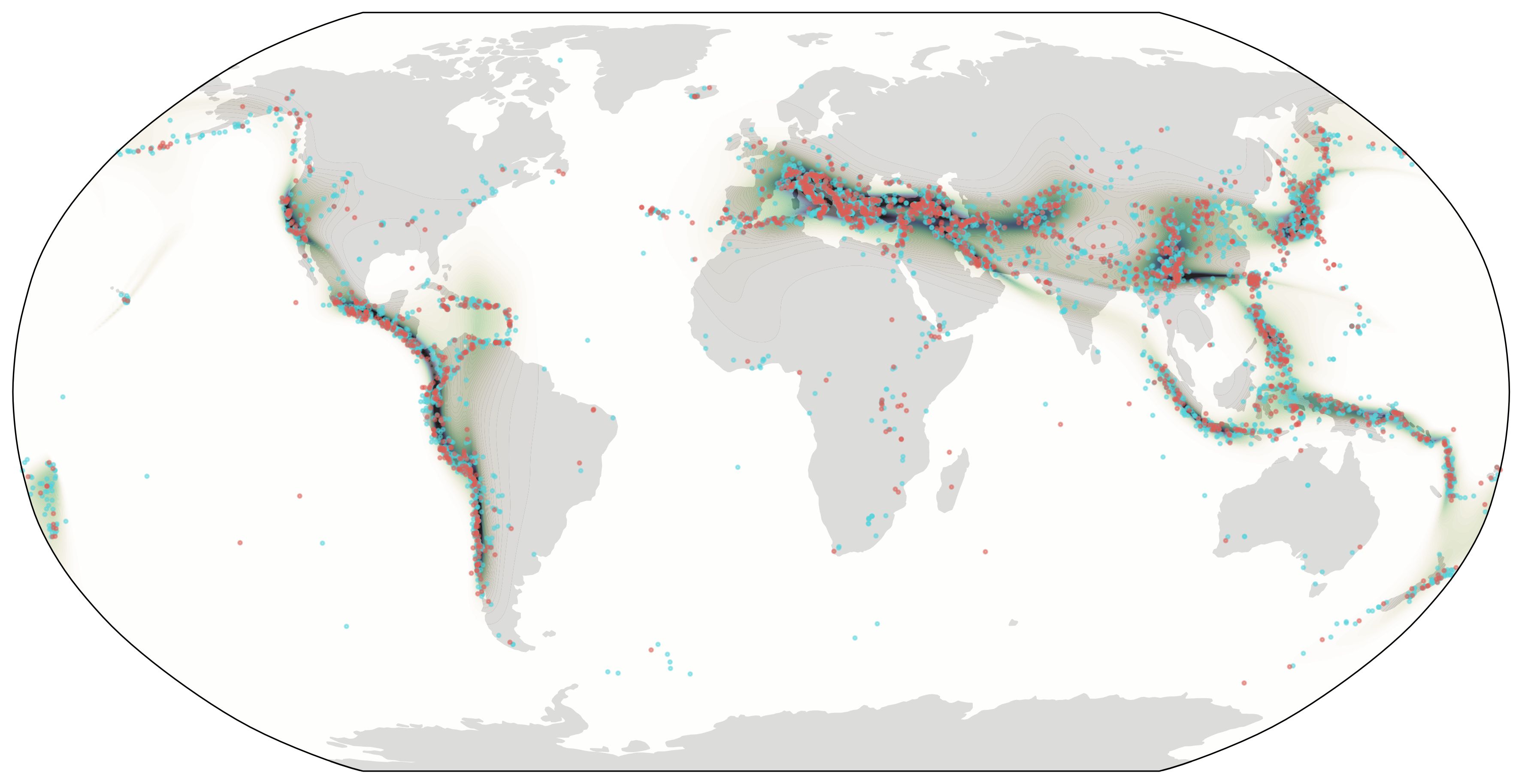}
  \put(-205,35){\rotatebox{90}{Earthquake}}
\end{subfigure}\hfil
\begin{subfigure}{.49\textwidth}
  \includegraphics[width=\linewidth]{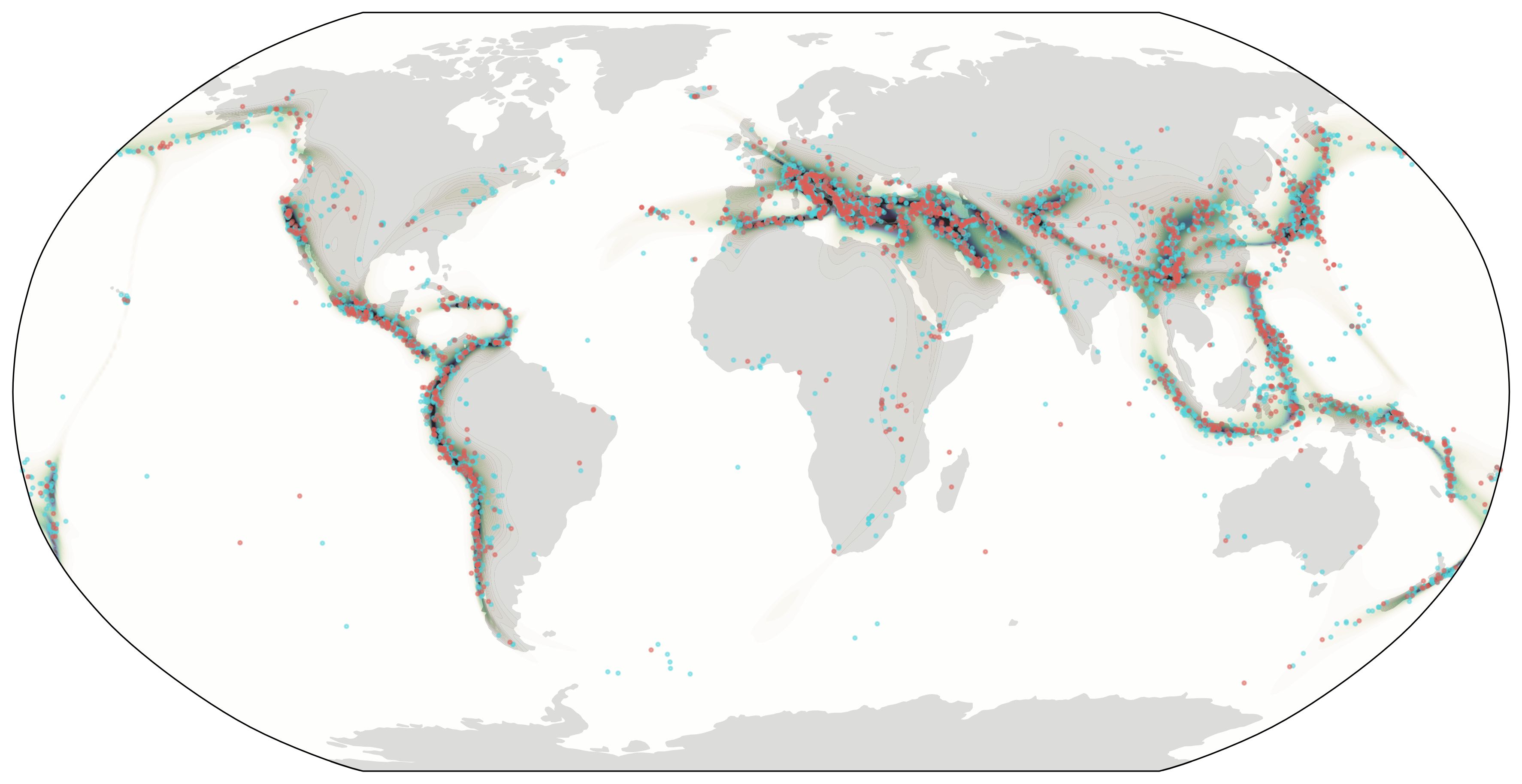}
\end{subfigure}\hfil
\begin{subfigure}{.49\textwidth}
  \includegraphics[width=\linewidth]{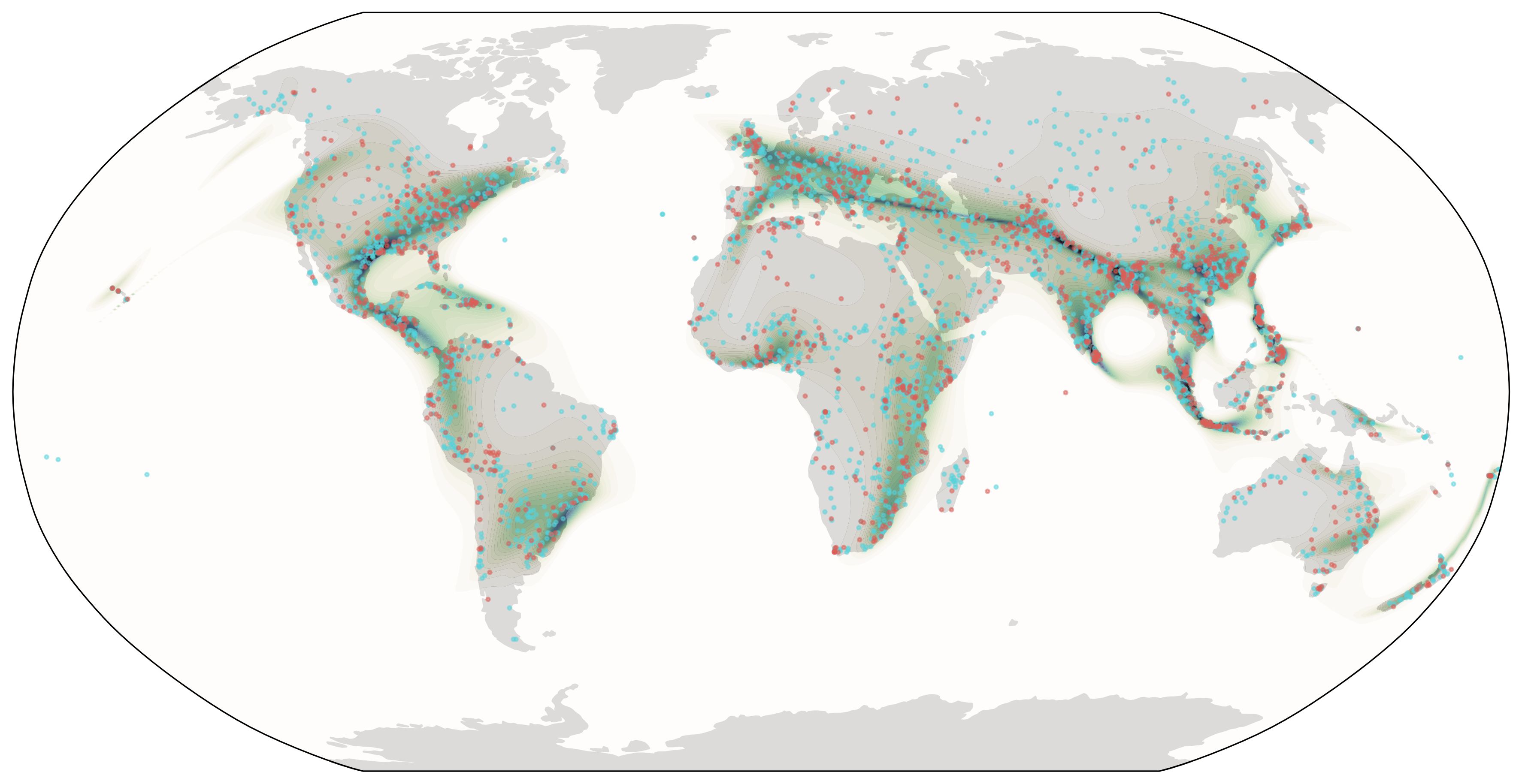}
  \put(-205,35){\rotatebox{90}{Flood}}
\end{subfigure}\hfil
\begin{subfigure}{.49\textwidth}
  \includegraphics[width=\linewidth]{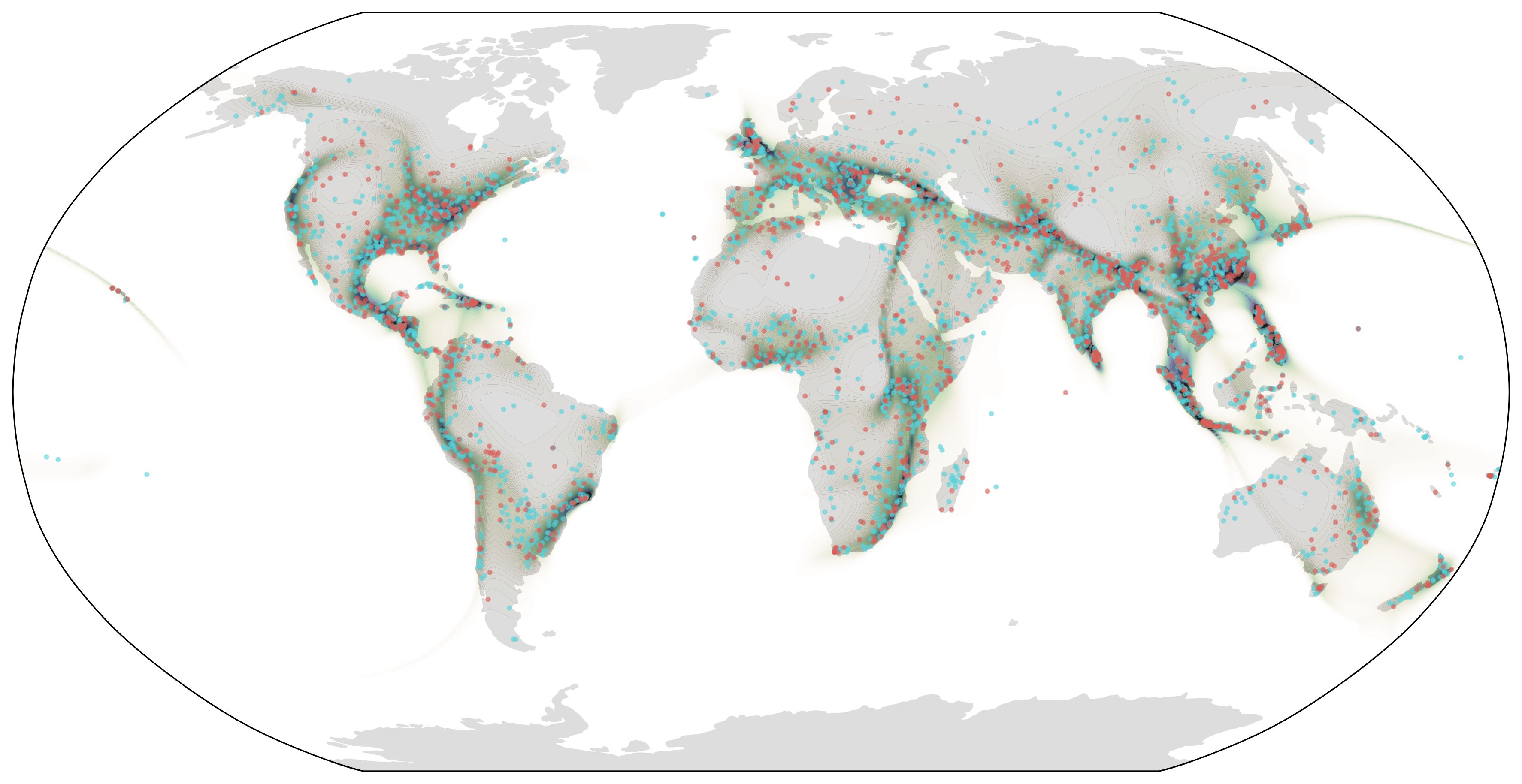}
\end{subfigure}\hfil
\begin{subfigure}{.49\textwidth}
  \includegraphics[width=\linewidth]{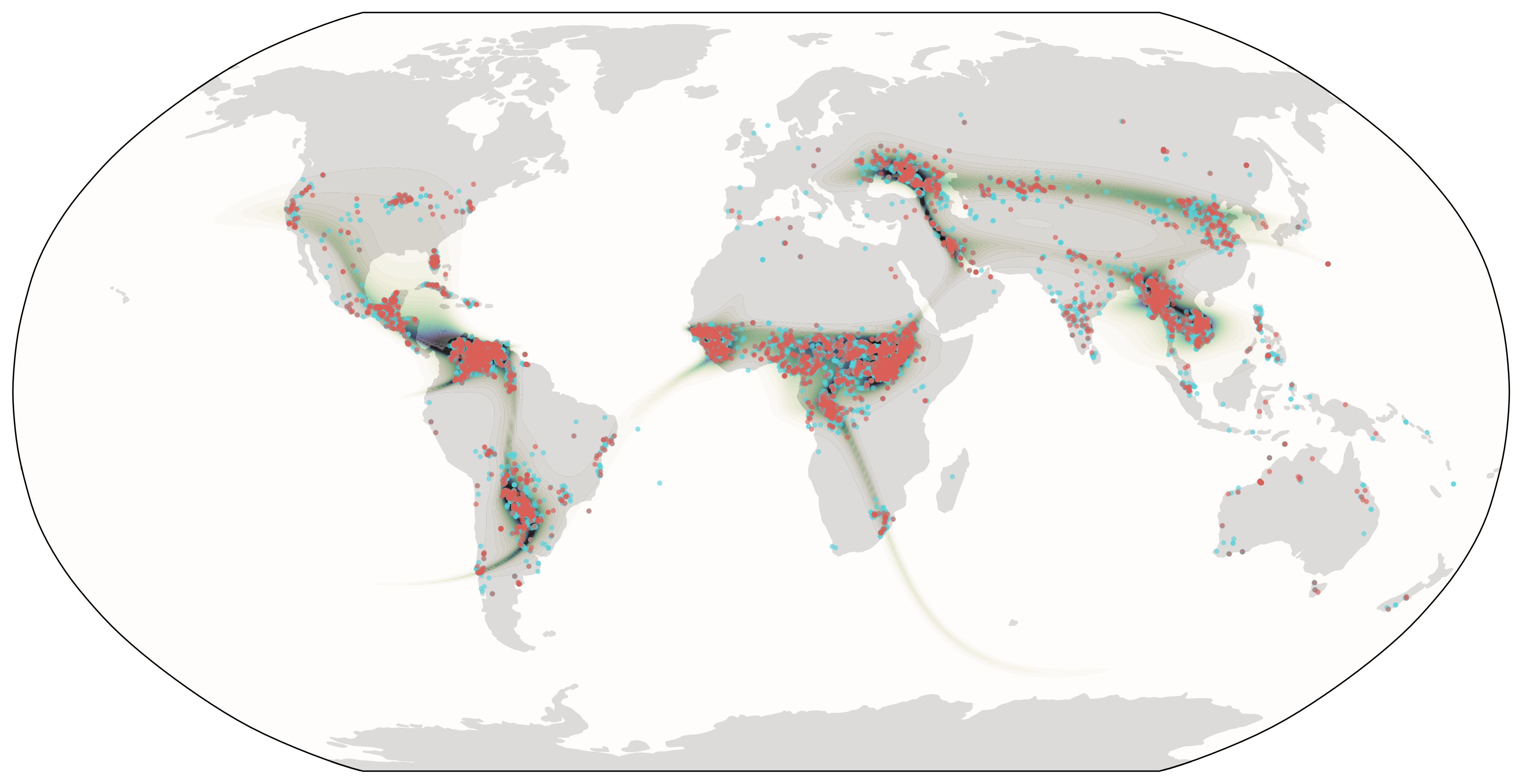}
  \put(-120,-10){Stereographic}
  \put(-205,35){\rotatebox{90}{Fire}}
\end{subfigure}\hfil
\begin{subfigure}{.49\textwidth}
  \includegraphics[width=\linewidth]{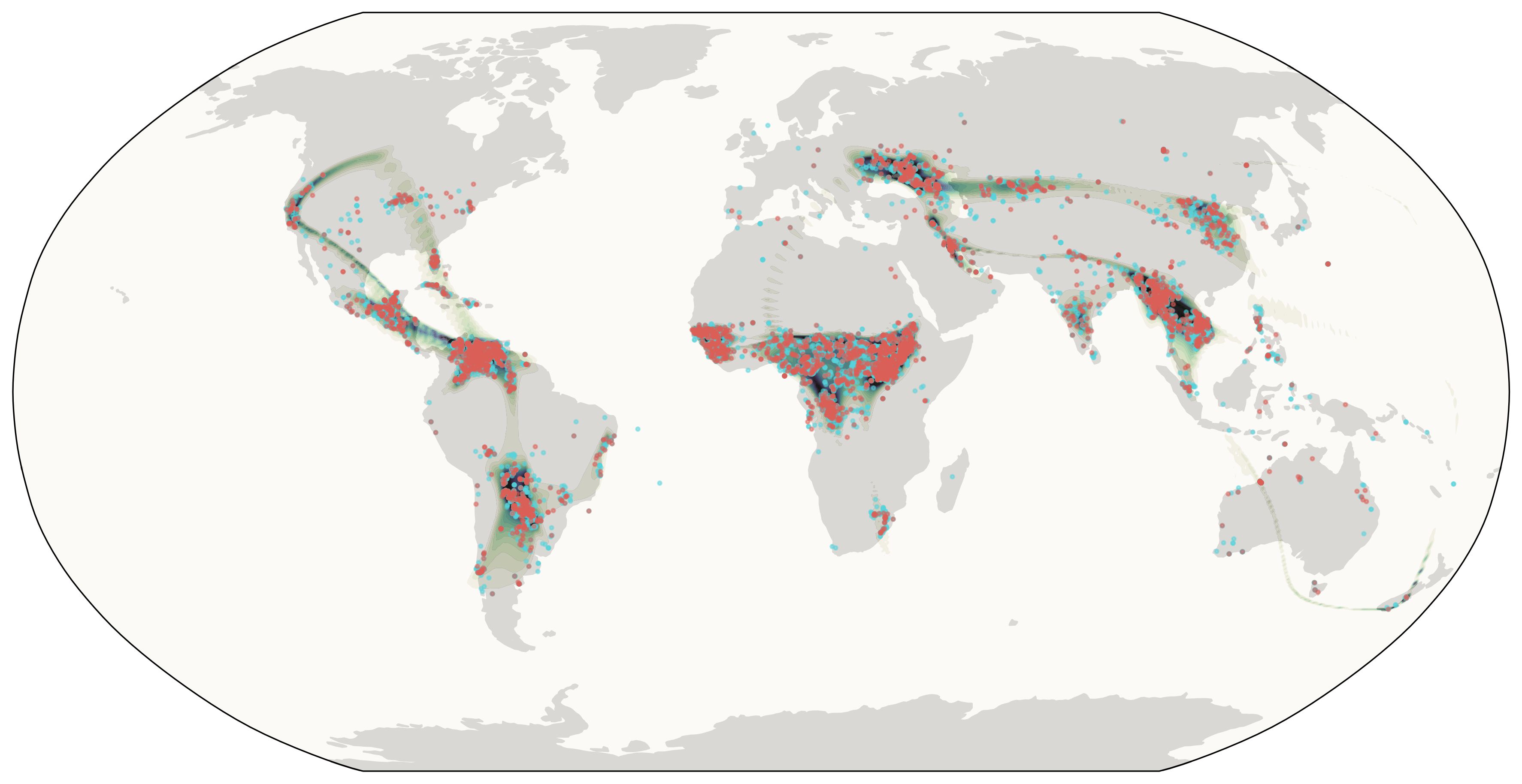}
  \put(-120,-10){Riemannian}
\end{subfigure}
  \caption{
    Density estimation for earth sciences data with Robinson projection.
    Blue and red dots represent training and testing datapoints, respectively.
    Heatmaps depict the log-likelihood of the trained models.
  }
  \label{fig:earth_robinson}
\end{figure}

\end{appendices}

\end{document}